\documentclass[conference]{IEEEtran}
\IEEEoverridecommandlockouts
\usepackage{cite}
\usepackage{amsmath,amssymb,amsfonts}
\usepackage{amsthm}
\usepackage{amsmath}
\usepackage{amssymb}
\usepackage{bbm}
\usepackage{subcaption}
\newtheorem{theorem}{Theorem}
\newtheorem{lemma}{Lemma}

\newtheorem{remark}{Remark}
\newtheorem{definition}{Definition}

\usepackage{algorithm}
\usepackage{algorithmic}
\usepackage{graphicx}
\usepackage[table]{xcolor}
\usepackage{makecell}
\usepackage{textcomp}
\usepackage{xcolor}
\usepackage{tikz}
\usepackage[spaces,hyphens]{url}

\usepackage{hyperref}
\usepackage{url}
\usetikzlibrary{arrows.meta,positioning,shapes.geometric}

\renewcommand{\baselinestretch}{0.9435}
\newcommand{\remove}[1]{}

\def\BibTeX{{\rm B\kern-.05em{\sc i\kern-.025em b}\kern-.08emT\kern-.1667em\lower.7ex\hbox{E}\kern-.125emX}}

\begin{document}

\title{Pay to Learn, Share to Earn: Incentivized Federated Multi-Player Bandits}

\author{Pavamana K J and Chandramani Singh, Indian Institute of Science, Bengaluru, India}

\maketitle

\begin{abstract}
Federated multi-player multi-armed bandit problems model collaborative sequential decision-making where multiple players interact with a common bandit environment and share information through a central server to accelerate learning. Existing federated bandit frameworks typically assume that all players willingly share their local observations with the server. However, this assumption is often unrealistic in practical settings where players are self-interested and may not participate in collaboration without explicit incentives. To address this challenge, we propose an incentive-aware federated bandit framework in which players receive rewards for sharing information with the server and incur costs when buying information from the server. We develop a UCB-based algorithm, termed Buying-UCB, that balances individual exploration and collaborative learning by incorporating both sharing incentives and information acquisition costs into the learning process. We theoretically analyze the proposed algorithm and derive upper bounds on the group regret and buying cost. Our analysis further characterizes the trade-off between fully collaborative federated learning and completely independent learning. Extensive numerical experiments validate the theoretical findings and demonstrate the effectiveness of the proposed framework under different collaboration and pricing regimes.
\end{abstract}


\section{Introduction}
\label{sec:related_work}

The multi-armed bandit (MAB) problem is a fundamental framework for sequential decision-making under uncertainty and has numerous applications, including recommendation systems \cite{chen2013combinatorial}, wireless communications \cite{gupta2018low}, online advertising \cite{tang2015personalized}, and resource allocation \cite{lattimore2020bandit}. In the classical MAB setting, a learning agent repeatedly selects one arm from a finite set, where each arm yields a random reward drawn from an unknown distribution. The objective is to maximize the cumulative reward, or equivalently minimize the cumulative regret, by balancing the exploration--exploitation trade-off. This trade-off has led to the development of widely used algorithms such as Upper Confidence Bound (UCB) \cite{auer2002finite,bubeck2012regret} and Thompson Sampling \cite{agrawal2012analysis}.

Many practical applications, however, involve multiple players making decisions simultaneously rather than a single learner \cite{6483307,besson2018multi}. For example, in an online shopping platform, multiple users repeatedly purchase products and observe individual satisfaction levels, which can be viewed as stochastic rewards. Over time, users gradually learn product preferences from their own purchase history, where users correspond to players and products correspond to arms.

In large-scale systems, directly sharing observations among users is often impractical due to privacy, communication, and scalability constraints. Instead, users can upload locally collected information to a central server, which aggregates the information and provides useful statistics or recommendations. This motivates the federated multi-player multi-armed bandit framework, where players collaborate through a central server without directly exchanging observations. Such collaboration significantly accelerates learning and improves overall performance. Following prior work, our objective is to maximize the cumulative reward of all players, equivalently minimizing the group regret while also minimizing the buying cost incurred when purchasing information from the server.

Although extensive work exists on multi-player MABs, most approaches assume that players willingly share their observations with a central server or neighboring players
\cite{11414221,karpov2024communication,wang2022achieving,zhu2025decentralized,madhushani2021call,chawla2020gossiping,wang2020optimal,bistritz2020cooperative,chakraborty2017coordinated,chen2023demand,liu2026distributed}. This assumption is often unrealistic since self-interested players may be unwilling to share information without explicit incentives from the server \cite{karimireddy2022mechanisms}. For example, in an online shopping platform, users purchase products and provide reviews or ratings that improve recommendations for future users. However, writing reviews requires additional effort while offering limited direct benefit to the contributor. Consequently, many users prefer to consume recommendations generated from others' feedback without contributing themselves, leading to free-riding behavior that reduces the amount of shared information and slows the overall learning process. Therefore, appropriate incentive mechanisms are essential for sustaining collaboration in federated bandit systems.

In such systems, the central server collects observations from multiple players, aggregates them, and redistributes useful information to improve learning. Since this information accelerates learning, accessing it naturally incurs a buying cost, while players contributing observations receive incentives. The server therefore generates revenue by selling aggregated information but also incurs expenses in rewarding contributors. Its objective is to maximize the net utility, defined as the revenue earned from selling information minus the incentive payments. In contrast, players aim to minimize their learning regret by effectively utilizing the purchased information.

As a concrete example, consider an online shopping platform where users repeatedly choose products from a large catalog. Users can share reviews, ratings, or purchase experiences with the platform, which aggregates this information to generate recommendations for other users. Since these recommendations help users identify better products more quickly, users may be willing to pay for them, while contributors receive incentives for sharing their feedback.

\subsection{Related Work}

Collaborative multi-agent multi-armed bandit (MAB) algorithms have been extensively studied to accelerate learning through information sharing among multiple agents. Existing research has primarily focused on designing efficient communication protocols that improve learning performance while reducing communication overhead.

Early work on cooperative and decentralized bandits investigated how agents can exchange information under limited communication.~\cite{chakraborty2017coordinated} studied the trade-off between coordinated and decentralized exploration under expensive communication.~\cite{chawla2020gossiping} proposed the Gossiping Insert-Eliminate algorithm, where agents communicate through pairwise gossip while exchanging only arm recommendations.~\cite{wang2020optimal} developed asymptotically optimal decentralized algorithms for multiplayer bandits, and~\cite{bistritz2020cooperative} investigated cooperative bandit optimization over time-varying communication networks. More recently,~\cite{zhu2025decentralized} proposed decentralized UCB algorithms that improve logarithmic regret through neighborhood information sharing, while~\cite{liu2026distributed} extended distributed bandits to heterogeneous agents communicating over Erd\H{o}s--R\'enyi random networks.

Another active research direction aims to reduce communication without sacrificing regret performance.~\cite{madhushani2021call} proposed ComEx, a communication protocol that selectively exchanges exploration information and achieves logarithmic communication complexity while maintaining the same regret order as full communication.~\cite{wang2022achieving} introduced UCB-TCOM, which simultaneously achieves near-optimal group and individual regret with only $O(\log\log T)$ communication rounds.~\cite{chen2023demand} further proposed an on-demand communication protocol for asynchronous multi-agent bandits, where communication frequency adapts to the empirical pulling rates of agents.~\cite{karpov2024communication} analyzed the communication--regret trade-off in collaborative learning, and~\cite{11414221} recently developed cooperative bandit algorithms that simultaneously achieve optimal regret and asymptotically optimal communication costs.

Although these works significantly improve collaborative learning efficiency, they all assume that agents willingly share their observations whenever required by the communication protocol. Consequently, communication is treated solely as a system overhead, while the incentives required to motivate self-interested participants are not explicitly modeled. In many practical federated systems, however, contributing observations requires effort or resources, whereas consuming aggregated information is relatively inexpensive, leading to free-riding behavior.

In contrast, our work introduces an incentive-aware federated multi-player bandit framework in which players receive rewards for sharing observations with a central server while paying a buying cost to acquire aggregated information. The proposed Buying-UCB algorithm jointly optimizes exploration and information acquisition decisions, explicitly modeling the economic interaction between the server and self-interested players. Our framework provides theoretical guarantees on both the group regret and the cumulative information buying cost.

\paragraph{Our Contributions}
The main contributions of this paper are summarized as follows.

\begin{enumerate}
    \item We propose a federated multi-player multi-armed bandit framework with incentive-aware collaboration, where players earn rewards for sharing information with a central server and incur a buying cost to acquire information from it.

    \item We develop \emph{Buying-UCB}, a UCB-based algorithm that adaptively balances individual exploration and collaborative information acquisition through a buying-cost-dependent policy.

    \item We derive finite-time upper bounds on the group pulling regret and buying cost, and establish that the buying complexity is $O(\log T)$ for the optimal arm and only $O(\log\log T)$ for every sub-optimal arm.

    \item We show that the proposed framework smoothly interpolates between fully collaborative federated learning and independent learning as the buying cost varies, thereby characterizing the trade-off between collaboration and information acquisition cost.

    \item We validate the theoretical results through extensive simulations, demonstrating the effectiveness of the proposed algorithm under different buying costs and numbers of players.
\end{enumerate}
\paragraph{Organization of the Paper}
The remainder of this paper is organized as follows. Section~\ref{sec:system_model} introduces the proposed incentive-aware federated bandit framework and formulates the problem. Section~\ref{sec:algorithm} presents the proposed Buying-UCB algorithm along with the player-specific confidence-set construction. Section~\ref{sec:regret_analysis} establishes the theoretical guarantees, including the group pulling regret and buying-cost analysis. Section~\ref{sec:numerical_results} presents numerical results demonstrating the effectiveness of the proposed algorithm under different collaboration settings. Finally, Section~\ref{sec:conclusion} concludes the paper and outlines future research directions.

\section{System Model}
\label{sec:system_model}
We consider a federated multi-player multi-armed bandit system consisting of $M$ players and a central server. Each player interacts with the same stochastic multi-armed bandit instance containing $N$ arms, indexed by $k \in \{1,2,\dots,N\}$. Each arm $k$ is associated with an unknown reward distribution having mean reward $\mu_k$. The reward distributions are assumed to be identical for all players.

At each time slot $t \in \{1,2,\dots,T\}$, every player selects one arm to pull. Let $r_{p,k}(t)$ denote the stochastic reward received by player $p$ when pulling arm $k$ at time $t$. The objective of each player is to maximize its cumulative reward over the time horizon $T$.


In addition to local learning, players can collaborate through a central server. Let $A_p(t)$ denote the arm selected by player $p$ at time $t$. Whenever a player shares its local observation $(A_p(t), r_{p,A_p(t)}(t))$ with the server, the player receives an incentive reward of $S$ from the server. We assume that all players share their observations after every arm pull and therefore receive a total sharing reward of $TS$ over the horizon $T$.\footnote{In this work, sharing is not modeled as a decision variable; we assume that players share their observations after every arm pull. Incorporating sharing as an explicit strategic decision, together with the associated incentive mechanism and participation decisions, is left as an important direction for future work.}

The server aggregates the information collected from all players and provides useful information upon request. Let $\gamma_{p,k}(t)$ be an indicator variable defined as

$$\gamma_{p,k}(t) =
\begin{cases}
1, & \text{if player } p \text{ buys information for arm } k\\
&\text{ at time } t, \\
0, & \text{otherwise}.
\end{cases}$$
\remove{\textcolor{blue}{where the server optimally determines the value of $B$. We assume that each player has an sufficient budget and can therefore buy information whenever required.}}

\remove{\footnote{\textcolor{blue}{Even though players have sufficient budget, they can not buy information at every time instant. This is because the regret definition explicitly includes the cumulative buying cost term. Hence, if a player buys information at all time instants, the player incurs a total buying cost of $BT$, which leads to linear regret. Therefore, players must carefully decide when buying information is beneficial, rather than buying information continuously. In fact, Theorem~\ref{thm:buying_regret} shows that the total number of buys made by a player is at most $\mathcal{O}(\ln T)$, and hence the regret due to buying information scales as $\mathcal{O}(B \ln T)$.}}}

Whenever player $p$ buys information from the server, the player incurs a buying cost of $B$. Hence, the total cumulative utility accumulated by player $p$ over the horizon $T$ is given by
\begin{equation}
\label{eqn: Player Reward}
   U_p(T)
=
\sum_{t=1}^{T} r_{p,A_p(t)}(t)
+
TS
-
B \sum_{t=1}^{T}\sum_{k=1}^{N}\gamma_{p,k}(t), 
\end{equation}

The objective of each player is to maximize its cumulative utility $U_p(T)$ by optimally balancing arm exploration, exploitation, and information buying decisions.

Similarly, the central server accumulates revenue by selling information to the players while paying incentives for information sharing. Specifically, whenever a player buys information from the server, the server earns a revenue of $B$, whereas the server pays an incentive of $S$ to players for every shared observation. Hence, the total cumulative revenue accumulated by the server over the time horizon $T$ is given by
\begin{equation}
\label{eqn:Server Revenue}
   U_s(T)
=
B \sum_{p=1}^{M}\sum_{t=1}^{T}\sum_{k=1}^{N}\gamma_{p,k}(t)
-
MTS. 
\end{equation}

The objective of the central server is to maximize its cumulative revenue
$U_s(T)$ by selecting an appropriate information buying price $B$. Since
the buying price directly influences the players' information purchasing
decisions, it determines both the server revenue and the effectiveness of
collaborative learning.\footnote{In this work, we treat the buying price
$B$ as an exogenously specified parameter and do not provide a mathematical
characterization of the optimal buying price. A theoretical formulation
and characterization of the optimal $B$ are left as future work. However,
the numerical results in Experiment~6 (Section~\ref{sec:numerical_results}), ``Server Profit vs.
Buying Cost,'' demonstrate that the server profit is non-monotonic in $B$
and that an intermediate buying cost maximizes the server's profit.}

The constrained optimization problem for all players can be formulated as follows:
\begin{equation}
\label{eqn:player_optimization}
\max_{\{A_{p}(t),\gamma_{p,k}(t) , \forall p,k,t\}}
\quad
\mathbb{E}\left[\sum_{p= 1}^{M}U_p(T)\right]
\end{equation}
subject to
\begin{equation}
\label{eqn:single_arm_constraint}
\sum_{k=1}^{N} \mathbf{1} \{A_{p} (t) = k\}  = 1,
\quad \forall p,t,
\end{equation}

\begin{equation}
\label{eqn:binary_buying_constraint}
\gamma_{p,k}(t) \in \{0,1\},
\quad \forall p,k,t,
\end{equation}

Let $\mu_* = \max_{k \in \{1,2,\dots,N\}} \mu_k$
denote the mean reward of the optimal arm. In the optimization problem defined in \eqref{eqn:player_optimization}, if the mean rewards $\mu_k$ of all arms were known beforehand, then the optimal strategy for each player would be to repeatedly select the optimal arm at every time slot. In such a case, there would be no uncertainty in the learning process, and consequently, the player would not require any additional information from the central server. Therefore, the optimal cumulative utility achievable by player $p$ over the horizon $T$ would be
\begin{equation}
\label{eqn:Best Reward}
    U_p^*(T) = T\mu_* + TS.
\end{equation}

However, in practice, the arm mean rewards $\mu_k$ are unknown and must be learned sequentially through interactions with the environment. As a result, players face the classical exploration--exploitation trade-off while also deciding whether buying information from the server is beneficial. To address this challenge, we propose a UCB-based algorithm that jointly incorporates learning, collaboration, and information acquisition decisions.

\section{An UCB-based Algorithm}
\label{sec:algorithm}
We now develop an algorithm to solve the optimization problem in \eqref{eqn:player_optimization}. Let $n_{p,k}(t)$ denote the number of times player $p$ has pulled arm $k$ up to time $t-1$. Since players may buy information from the central server, the learning process of a player is also influenced by the observations collected by other players. To capture this effect, we define $N_{p,k}(t)$ as the effective number of observations of arm $k$ available to player $p$ at time $t$. For player $p$ and arm $k$, let
\begin{equation}
\label{eqn:last_buying_time_before}
\tau_{p,k}^{t}
=
\max
\left\{
s \le t :
\gamma_{p,k}(s)=1
\right\},
\end{equation}
with the convention $\tau_{p,k}^{t}=0$
if no such purchase exists.
The effective number of observations available to player
$p$ immediately before the buying decision at time $t$ is
\begin{equation}
\label{eqn:effective_count_before_buy}
N_{p,k}(t)
=
n_{p,k}(t)
+
\sum_{q \neq p}
n_{q,k}
\!\left(
\tau_{p,k}^{t-1}
\right).
\end{equation}
Thus, $N_{p,k}(t)$ represents the information available
to player $p$ prior to any purchase at time $t$.
After the buying decision at time $t$, the effective
number of observations becomes
\begin{equation}
\label{eqn:effective_count_after_buy}
N_{p,k}^{+}(t)
=
N_{p,k}(t)
+
\sum_{q \neq p} n_{q,k}(t)
-
\sum_{q \neq p} n_{q,k}(\tau_{p,k}^{t-1}).
\end{equation}
Hence, $N_{p,k}^{+}(t)$ denotes the effective sample count
after incorporating any newly purchased observations at
time $t$.
Observe that if no purchase occurs at time $t$, then $N_{p,k}^{+}(t)
=
N_{p,k}(t).$ $\sum_{q \neq p} n_{q,k}(\tau_{p,k}^{t})$ represents the total number of observations of arm $k$ collected by all other players up to time $\tau_{p,k}^{t} - 1$.

Accordingly, the empirical mean estimate of arm $k$ maintained by player $p$ at time $t$ is defined as
\begin{align}
\label{eqn: Empirical Mean Estimate}
\hat{\mu}_{p,k}(t)
&=
\frac{1}{N_{p,k}^{+}(t)}
\Bigg(
\sum_{s=1}^{t-1}
r_{p,A_p(s)}(s)\mathbf{1}\{A_p(s)=k\}
\nonumber\\
&\quad +
\sum_{q \neq p}\sum_{s=1}^{\tau_{p,k}^{t}}
r_{q,A_q(s)}(s)\mathbf{1}\{A_q(s)=k\}
\Bigg).
\end{align}

Thus, the estimate maintained by a player depends not only on its locally collected observations but also on the information bought from the server. 

\subsection{Arm Pulling Decisions}
The arm selection decision at time $t$ is based on the Upper Confidence Bound (UCB) index. Specifically, player $p$ computes the UCB index of arm $k$ as
\begin{equation}
\label{eq:ucb_index}
\text{UCB}_{p,k}(t)
=
\hat{\mu}_{p,k}(t)
+
\sqrt{
\frac{2\log t}{N_{p,k}^{+}(t)}
}.
\end{equation}

The first term represents the empirical mean estimate of arm $k$, while the second term denotes the confidence interval that captures the uncertainty associated with the estimate. Intuitively, arms with fewer effective observations receive larger confidence bonuses, thereby encouraging exploration. Furthermore, since the empirical estimate and confidence interval depends on the observations collected by other players up to the latest information buying time $\tau_{p,k}^{t}$, the past arm pulls of other players also influence the arm selection decision of player $p$ at time $t$. Let us denote the confidence interval as follows 
\begin{equation}
\label{eqn:Confidence Interval}
    C_{p,k}(t) = \sqrt{\frac{2\log t}{N_{p,k}^{+}(t)}}
\end{equation}

At each time step $t$, player $p$ selects the arm with the highest UCB index, i.e.,
\begin{equation}
\label{eq:arm_selection}
A_p(t)
=
\arg\max_{k \in \{1,2,\dots,N\}}
\text{UCB}_{p,k}(t).
\end{equation}

\subsection{Server Data Collection}
At each time step $t$, every player $p$ shares its local observation $(A_p(t), r_{p,A_p(t)}(t))$
with the central server and receives an incentive reward of $S$. Consequently, the server maintains the complete history of arm pulls and rewards, including $n_{q,k}(t)$ and the corresponding reward observations $r_{q,k}(t)$ for all players $q$, arms $k$, and time instants $t$.

However, to avoid information leakage, the server does not directly broadcast the reward observations of other players. Instead, the server only broadcasts the pull statistics $n_{q,k}(t)$ for all players and arms. Using this information, each player estimates whether buying additional information from the server would significantly reduce the uncertainty in its reward estimates.

\subsection{The Buying Decision}

The key observation is that collaboration is most beneficial for arms whose optimality has not yet been determined. Once an arm is identified with high confidence as the optimal arm, repeatedly buying additional observations of that arm provides little improvement in the pulling regret while continuing to increase the buying cost. Motivated by this observation and inspired by \cite{wang2022achieving}, we restrict information buying to only those arms whose optimality is still uncertain.

To identify such arms, each player constructs confidence intervals using its effective observations. Specifically, the lower and upper confidence bounds of player $p$ for arm $k$ at time $t$ are defined as
\begin{align}
\mathrm{tLCB}_{p,k}(\alpha,t)
&=
\hat{\mu}_{p,k}(t)
-
\alpha
\sqrt{\frac{\log t}{N_{p,k}^{+}(t)}},
\label{eq:lcb}
\\
\mathrm{tUCB}_{p,k}(\alpha,t)
&=
\hat{\mu}_{p,k}(t)
+
\alpha
\sqrt{\frac{\log t}{N_{p,k}^{+}(t)}},
\label{eq:ucb_confidence}
\end{align}
where $\alpha>0$ controls the confidence width. Using these confidence bounds, player $p$ constructs the confidence set
\begin{align}
\mathcal{C}_{p,t}(\alpha)
=
\left\{
k:
\mathrm{tLCB}_{p,k}(\alpha,t)
<
\max_{k'\neq k}
\mathrm{tUCB}_{p,k'}(\alpha,t)
\right\}.
\label{eq:confidence_set}
\end{align}
The confidence set contains only those arms that cannot yet be statistically distinguished from the current best arm. Consequently, player $p$ considers buying information only for arms belonging to $\mathcal{C}_{p,t}(\alpha)$. For every arm $k\in\mathcal{C}_{p,t}(\alpha)$, the decision to buy information is based on the potential reduction in the confidence interval. Let

\[
C_{p,k}^{\mathrm{No\,Buy}}(t)
=
\sqrt{\frac{2\log t}{N_{p,k}(t)}}
\]

denote the confidence interval before buying, and let

\[
C_{p,k}^{\mathrm{Buy}}(t)
=
\sqrt{
\frac{
2\log t
}{
N_{p,k}(t)
+
\sum_{q\neq p}n_{q,k}(t)
-
\sum_{q\neq p}n_{q,k}(\tau_{p,k}^{t-1})
}
}
\]
denote the confidence interval after incorporating all newly available observations from the server. Player $p$ purchases information only when the reduction in uncertainty is sufficiently large relative to the buying cost $B$, namely,

\begin{align}
\frac{
C_{p,k}^{\mathrm{No\,Buy}}(t)
}{
C_{p,k}^{\mathrm{Buy}}(t)
}
\ge
\sqrt{B+1}.
\label{eqn:Confidence Ratio}
\end{align}
The threshold increases with the buying cost. Consequently, larger values of $B$ require a greater reduction in uncertainty before information is purchased, whereas smaller values of $B$ encourage more frequent collaboration.
Simplifying \eqref{eqn:Confidence Ratio} yields the equivalent buying condition

\begin{align}
\sum_{q\neq p}n_{q,k}(t)
-
\sum_{q\neq p}n_{q,k}(\tau_{p,k}^{t-1})
\ge
B\,N_{p,k}(t).
\label{eq:buying_condition}
\end{align}
Therefore, player $p$ buys information regarding arm $k$ only if $k\in\mathcal{C}_{p,t}(\alpha)$ and the buying condition in \eqref{eq:buying_condition} is satisfied. In this case, the player pays the buying cost $B$, and the server shares all reward observations of arm $k$ collected by the other players during the interval $(\tau_{p,k}^{t-1},t)$,
\begin{equation}
\label{eqn:Server Information}
\Bigl\{
r_{q,A_q(s)}(s):
q\neq p,\;
\tau_{p,k}^{t-1} \le s < t ,\;
A_q(s)=k
\Bigr\}.
\end{equation}
Equivalently, the player receives $\sum_{q\neq p}
\left(
n_{q,k}(t)
-
n_{q,k}(\tau_{p,k}^{t-1})
\right)$ new observations. Thus, only observations that have not been previously purchased are transferred, after which the player updates its effective sample count and empirical reward estimate.

The complete Buying-UCB procedure is summarized in Algorithm~\ref{alg:buying_ucb}. The sequence of operations performed by a player during a generic round $t$ is illustrated in Fig.~\ref{fig:algorithm_timeline}. 
\begin{figure}[t]
\centering
\begin{tikzpicture}[
    thick,
    >=stealth,
    every node/.style={font=\footnotesize}
]

\draw[->,blue] (0,0) -- (7,0);

\node at (0.4,1.15) {\Huge ...};
\node at (6.6,1.15) {\Huge ...};

\draw[dashed] (1.2,-0.15) -- (1.2,2.45);
\draw[dashed] (5.8,-0.15) -- (5.8,2.45);

\node at (1.2,-0.35) {$t$};
\node at (6.0,-0.35) {$t+1$};

\draw[<->,red] (1.2,2.25) -- (5.8,2.25);
\node[above] at (3.5,2.25) {Round $t$};

\node[
    fill=cyan!20,
    align=left,
    text width=4.0cm
] at (3.5,1.55)
{
\textbf{1.} For arms in
$\mathcal{C}_{p,t-1}(\alpha)$,\\[1mm]
check the buying condition\\
\textbf{2.} Buy information (if satisfied)
};

\node[
    fill=green!20,
    align=left,
    text width=4.0cm
] at (3.5,0.65)
{
\textbf{3.} Update $N_{p,k}^{+}(t)$ and $\hat{\mu}_{p,k}(t)$
};

\node[
    fill=orange!20,
    align=left,
    text width=4.0cm
] at (3.5,-0.15)
{
\textbf{4.} Compute UCBs, pull an arm and update $\mathcal{C}_{p,t}(\alpha)$ 
};

\end{tikzpicture}
\caption{Operations performed by player $p$ during a generic round $t$. Information is purchased only for arms in the previous confidence set $\mathcal{C}_{p,t-1}(\alpha)$. After updating the effective statistics, the player computes the UCB indices, selects an arm, uploads the observed reward to the server, and constructs the confidence set $\mathcal{C}_{p,t}(\alpha)$ for the next round.}
\label{fig:algorithm_timeline}
\end{figure}
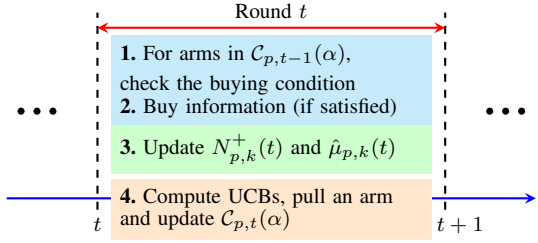

\begin{algorithm}[t]
\caption{Buying-UCB with Confidence Set}
\label{alg:buying_ucb}

\begin{algorithmic}[1]

\STATE \textbf{Input:} Number of players $M$, number of arms $N$, horizon $T$, buying cost $B$

\STATE Initialize
$n_{p,k}(0)$,
$N_{p,k}^{+}(0)$,
$\hat{\mu}_{p,k}(0)$,
bought counts/rewards,
$\tau_{p,k}^{0}=0$, and
$\gamma_{p,k}(t)=0,\ \forall t\le N,\ \forall p,k$.

\STATE Initialize
$\mathcal{C}_{p,N}(\alpha)=\{1,\ldots,N\}$,
$\forall p$.

\STATE Each player pulls every arm once during the first $N$ rounds.

\FOR{$t=N+1$ to $T$}

    \STATE Server broadcasts pull counts
    $n_{p,k}(t)$ for all $p,k$.

    \FOR{each player $p$}

        \FOR{each arm $k\in\mathcal{C}_{p,t-1}(\alpha)$}

            \STATE Set $\tau_{p,k}^{t-1}$ according to
            \eqref{eqn:last_buying_time_before}.

            \IF{condition in \eqref{eq:buying_condition} is satisfied}

                \STATE Player $p$ buys information for arm $k$.

                \STATE Server shares the information \eqref{eqn:Server Information}
                
                \STATE Update bought counts and rewards.

                \STATE Set
                $\gamma_{p,k}(t)=1$.

            \ELSE

                \STATE Set
                $\gamma_{p,k}(t)=0$.

            \ENDIF

        \ENDFOR

        \FOR{each arm $k$}

            \STATE Compute
            $N_{p,k}^{+}(t)$
            using \eqref{eqn:effective_count_after_buy}.

            \STATE Compute
            $\hat{\mu}_{p,k}(t)$
            using \eqref{eqn: Empirical Mean Estimate}.

            \STATE Compute
            $\mathrm{UCB}_{p,k}(t)$
            using \eqref{eq:ucb_index}.

        \ENDFOR

        \STATE Select arm
        $A_p(t)$ according to
        \eqref{eq:arm_selection}.

        \STATE Pull arm $A_p(t)$ and observe reward
        $r_{p,A_p(t)}(t)$.

        \STATE Send
        $(A_p(t),r_{p,A_p(t)}(t))$
        to the server.

        \STATE Update
        $n_{p,A_p(t)}(t)$
        and
        $\hat{\mu}_{p,A_p(t)}(t)$.

        \STATE Construct the confidence set
        $\mathcal{C}_{p,t}(\alpha)$
        according to
        \eqref{eq:confidence_set}.

    \ENDFOR

\ENDFOR

\end{algorithmic}
\end{algorithm}

\begin{remark}
The confidence-set construction in \eqref{eq:confidence_set} is inspired by the communication set proposed in \cite{wang2022achieving}. However, unlike \cite{wang2022achieving}, where all players maintain synchronized statistics and therefore share a common communication set, the proposed Buying-UCB algorithm constructs a \emph{player-dependent} confidence set. Since players buy information asynchronously and may purchase different observations at different time instants, they generally maintain different effective sample counts and empirical reward estimates. Consequently, each player constructs its own confidence set based on its local learning state, resulting in an asynchronous and decentralized confidence-set construction.
\end{remark}

\section{Regret Analysis}
\label{sec:regret_analysis}
In this section, we analyze the regret performance of Algorithm~\ref{alg:buying_ucb}. The regret of player $p$ over the time horizon $T$ is defined as the difference between the optimal cumulative utility and the utility achieved by the proposed algorithm, i.e.,
\begin{equation}
\label{eq:regret_definition}
R_p(T)
=
U_p^*(T)-U_p(T).
\end{equation}
Substituting the expressions for $U_p^*(T)$ and $U_p(T)$ from \eqref{eqn:Best Reward} and \eqref{eqn: Player Reward} respectively, we obtain
\begin{align}
R_p(T)
&=
\underbrace{
T\mu_*
-
\sum_{t=1}^{T}
r_{p,A_p(t)}(t)
}_{R_p^{\mathrm{pull}}(T)}
+
\underbrace{
B
\sum_{t=1}^{T}
\sum_{k=1}^{N}
\gamma_{p,k}(t)
}_{R_p^{\mathrm{buy}}(T)}.
\label{eq:regret_expanded}
\end{align}

where $R_p^{\mathrm{pull}}(T)$ denotes the \emph{individual pulling regret} incurred by player $p$, and $R_p^{\mathrm{buy}}(T)$ denotes the \emph{individual buying cost}.
Accordingly, the regret analysis is divided into two parts:
\begin{enumerate}
    \item Bounding the pulling regret incurred due to selecting sub-optimal arms.
    \item Bounding the cumulative buying cost incurred while purchasing information from the server.
\end{enumerate}

\subsection{Group Regret and Group Buying Cost}
The group pulling regret is defined as the sum of the individual pulling regret $R_{p}^{\mathrm{pull}}(T)$
of all players:
\begin{align}
R_{\mathrm{group}}^{\mathrm{pull}}(T)
&=
\sum_{p=1}^{M}
R_p^{\mathrm{pull}}(T)
\nonumber\\
&=
MT\mu_*
-
\sum_{p=1}^{M}
\sum_{t=1}^{T}
r_{p,A_p(t)}(t).
\label{eq:group_pulling_regret}
\end{align}
Similarly, the group buying cost is defined as the sum of the individual buying
cost $R_{p}^{\mathrm{buy}}(T)$ incurred by all players:
\begin{align}
R_{\mathrm{group}}^{\mathrm{buy}}(T)
&=
\sum_{p=1}^{M}
R_p^{\mathrm{buy}}(T)
\nonumber\\
&=
B
\sum_{p=1}^{M}
\sum_{t=1}^{T}
\sum_{k=1}^{N}
\gamma_{p,k}(t).
\label{eq:group_buying_regret}
\end{align}
The total group regret is therefore
\begin{align}
R_{\mathrm{group}}(T)
&=
R_{\mathrm{group}}^{\mathrm{pull}}(T)
+
R_{\mathrm{group}}^{\mathrm{buy}}(T)
\nonumber\\
&=
\sum_{p=1}^{M} R_p(T).
\label{eq:group_regret}
\end{align}
The buying cost parameter $B$ plays an important role in determining the collaboration behavior of players. If $B=0$, then the buying condition in \eqref{eq:buying_condition} is satisfied at every time instant, implying that players buy information from the server continuously. This corresponds to the standard federated bandit setting in which all players fully collaborate through the server. In this case, the group pulling regret can be shown to scale as
$$
R_{\mathrm{group}}^{\mathrm{pull}}(T)=O\left(\log T\right).
$$ A complete proof of the regret bound for the special case $B=0$ is provided in the supplementary material.
On the other hand, if $B \rightarrow \infty$, then the buying condition in \eqref{eq:buying_condition} is never satisfied, and therefore players never buy information from the server. In this setting, each player relies entirely on its own exploration without any collaboration, which effectively reduces the problem to $M$ independent single-player bandit problems. Consequently, the group pulling regret scales as
$$
R_{\mathrm{group}}^{\mathrm{pull}}(T)=O\left(M \log T\right).
$$

Therefore, our objective is to design an algorithm that achieves group pulling regret as close as possible to the ideal fully collaborative regret scaling of $
O\left(\log T\right),$ while simultaneously accounting for the information buying costs.

Before presenting the main regret result, we define the sub-optimality gap of arm $k$ as $
\Delta_k = \mu_* - \mu_k,
$
where $\mu_*$ denotes the mean reward of the optimal arm. Let us define group pull counts upto time $t-1$ as follows
\begin{equation}
n_k^{\mathrm{group}}(t)
=
\sum_{p=1}^{M}
n_{p,k}(t),
\qquad  \forall t ,k=1,\ldots,N.
\label{eq:group_pull_count}
\end{equation}

\begin{lemma}
\label{lem:group_pulling_gap}
The expected group pulling regret satisfies

\begin{align}
\mathbb{E}
\!\left[
R_{\mathrm{group}}^{\mathrm{pull}}(T)
\right]
=
\sum_{k:\Delta_k>0}
\Delta_k\,
\mathbb{E}
\!\left[
n_k^{\mathrm{group}}(T+1)
\right].
\label{eq:group_pulling_gap}
\end{align}

\end{lemma}

\begin{proof}
    Please refer our supplementary material.
\end{proof}

\begin{theorem}
\label{thm:group_pulling_regret}
For $\alpha > 1/\sqrt{2}$, the expected group pulling regret of
Algorithm~\ref{alg:buying_ucb} satisfies
\begin{align}
\mathbb{E}
\!\left[
R_{\mathrm{group}}^{\mathrm{pull}}(T)
\right]
&\le
\sum_{k:\Delta_k>0}
\Delta_k
\Bigg[
\min\{M,B+1\} \times \nonumber \\
& \quad \quad \left(
\frac{8\log T}{\Delta_k^2}
+2
\right)+ \frac{4MN\alpha^2}{2 \alpha^2-1}+
\frac{2M}{3}
\Bigg].
\label{eq:group_pulling_regret_bound}
\end{align}
\end{theorem}
\begin{proof}
    Please refer our supplementary material.
\end{proof}

\begin{remark}
Unlike \cite{wang2022achieving}, whose regret analysis requires
$1<\alpha<\sqrt{2}$, Theorem~\ref{thm:group_pulling_regret} is valid for any
$\alpha>\frac{1}{\sqrt{2}}$, thereby allowing a wider range of confidence parameters.
\end{remark}

\begin{remark}
\label{remark:group_pulling}

Theorem~\ref{thm:group_pulling_regret} shows that the buying cost controls the degree of collaboration through the factor $\min\{M,B+1\}$. When $B+1<M$, the dominant regret term is $O((B+1)\log T)$ and is independent of the number of players, indicating efficient collaboration. In contrast, when $B+1\ge M$, information sharing becomes infrequent, $\min\{M,B+1\}=M$, and the group pulling regret scales as $O(M\log T)$. Thus, Buying-UCB smoothly interpolates between a highly collaborative regime and an information acquisition-limited regime.
\end{remark}

\begin{remark}[Individual Regret]

Deriving a tight theoretical bound on the individual regret remains challenging because players buy information asynchronously, maintain different effective sample counts, and therefore possess different learning statistics. In contrast, UCB-TCOM \cite{wang2022achieving} keeps all players synchronized, enabling the group regret to be evenly distributed among players. Although a theoretical characterization remains open, Experiment~3 in Section~\ref{sec:numerical_results} shows that the individual regret decreases as the number of collaborating players increases.

\end{remark}
\begin{lemma}
\label{lem:suboptimal_buy_count}
Let
$B_{p,k}(T)$ denote the number of times player $p$
purchases information regarding arm $k$ up to time $T$.
For every sub-optimal arm $k$ with $\Delta_k>0$ and for $\alpha > 1/\sqrt{2}$, define

\begin{align*}
U_k(T)
=
\frac{
\min\{M,B+1\}
\left(
\dfrac{8\log T}{\Delta_k^2}+1
\right)+\dfrac{4MN\alpha^2}{2 \alpha^2-1}
+\dfrac{2M}{3}
}{M}.
\end{align*}

Then

\begin{align*}
\mathbb{E}[B_{p,k}(T)]
\le
1+
\frac{\log U_k(T)}
{\log(1+B)}.
\end{align*}

\end{lemma}
\begin{proof}
    Please refer our supplementary material.
\end{proof}

\begin{lemma}
\label{lem:optimal_buy_count}

Let $k^\star$ denote the optimal arm.
Then for $\alpha > 1/\sqrt{2}$

\begin{align*}
\mathbb{E}[B_{p,k^\star}(T)]
\le
1+
\frac{\log T}{\log(1+B)}.
\end{align*}

\end{lemma}
\begin{proof}
    Please refer our supplementary material.
\end{proof}
\begin{theorem}
\label{thm:buying_regret}
Under Algorithm~\ref{alg:buying_ucb},for $\alpha > 1/\sqrt{2}$, the expected buying
cost of player $p$ satisfies
\begin{align}
\mathbb{E}
\!\left[
R_p^{\mathrm{buy}}(T)
\right]
\le
&
B
\left(
1+
\frac{\log T}{\log(1+B)}.
\right)
\nonumber\\
&+
(N-1)B
\left(
1+
\frac{\log U_{\max}(T)}
{\log(1+B)}
\right),
\end{align}
where $U_{\max}(T)=\max_{k:\Delta_k>0}U_{k}(T).$
Furthermore, $\mathbb{E}
\!\left[
R_p^{\mathrm{buy}}(T)
\right]
=
O
\!\left(
B\log T
\right),$
and $\mathbb{E}
\!\left[
R_{\mathrm{group}}^{\mathrm{buy}}(T)
\right]
=
O
\!\left(
MB\log T
\right).$
\end{theorem}

\begin{proof}
    Please refer our supplementary material.
\end{proof}


\begin{theorem}
\label{thm:group_total_regret}

The total expected group regret of
Algorithm~\ref{alg:buying_ucb} for $\alpha > 1/\sqrt{2}$ satisfies 
\begin{align*}
   \mathbb{E}
\!\left[
R_{\mathrm{group}}(T)
\right]
=
O\!\left(
N\min\{M,B+1\}\log T+
MB\log T
\right). 
\end{align*}

\end{theorem}
\begin{proof}
    Please refer our supplementary material.
\end{proof}

\section{Numerical Results}
\label{sec:numerical_results}

\begin{figure*}[t]
    \centering

    \begin{subfigure}[t]{0.32\textwidth}
        \centering
        \includegraphics[width=\linewidth]{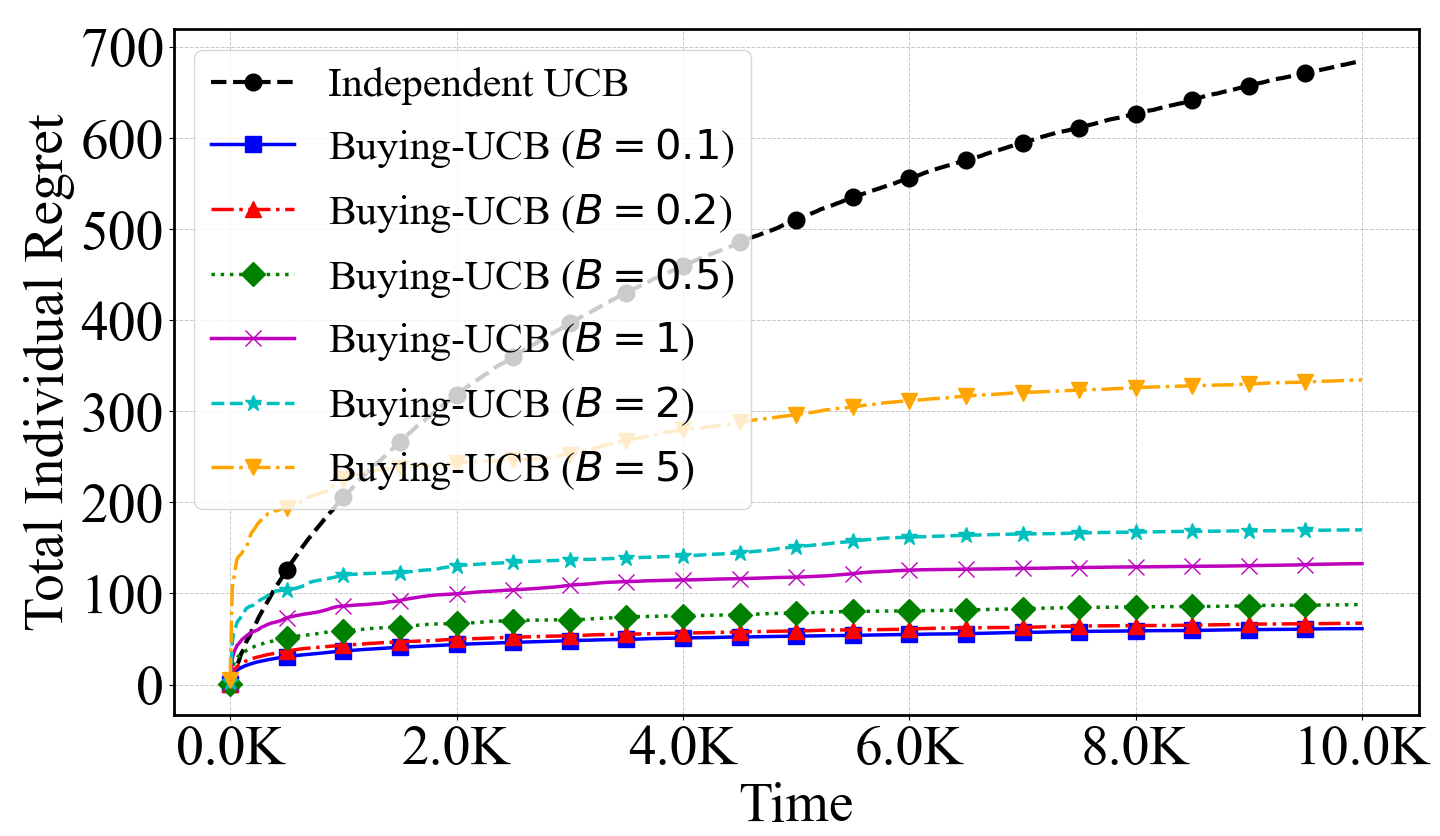}
        \caption{Total Individual Regret vs $T$ for different $B$.}
        \label{fig:regret_vs_time}
    \end{subfigure}
    \hfill
    \begin{subfigure}[t]{0.32\textwidth}
        \centering
        \includegraphics[width=\linewidth]{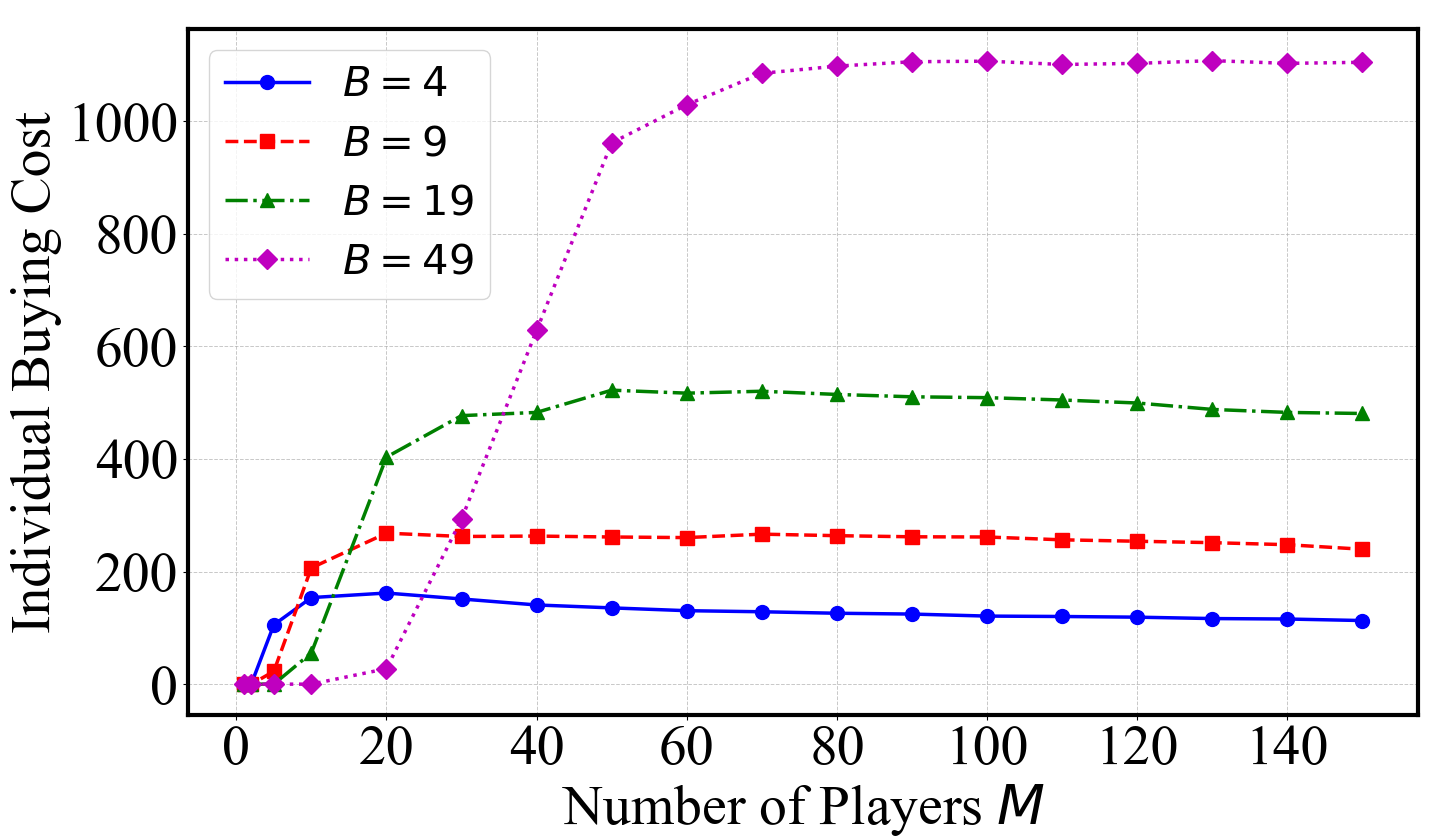}
        \caption{Individual buying cost vs $M$.}
        \label{fig:regret_vs_M}
    \end{subfigure}
    \hfill
    \begin{subfigure}[t]{0.32\textwidth}
        \centering
        \includegraphics[width=\linewidth]{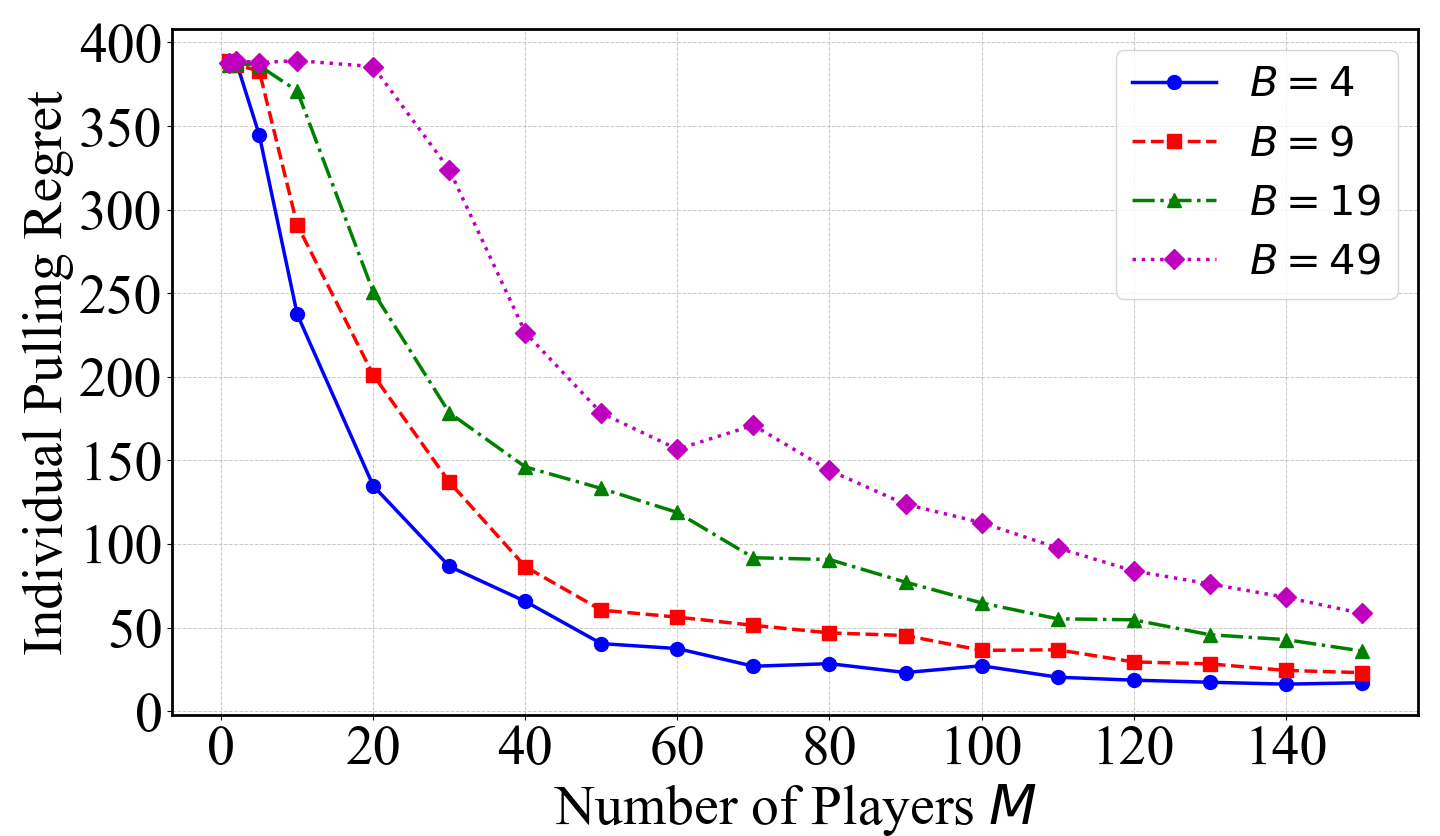}
        \caption{Individual pulling regret vs $M$.}
        \label{fig:pull_regret_vs_M}
    \end{subfigure}

    \vspace{0.3cm}

    \begin{subfigure}[t]{0.32\textwidth}
        \centering
        \includegraphics[width=\linewidth]{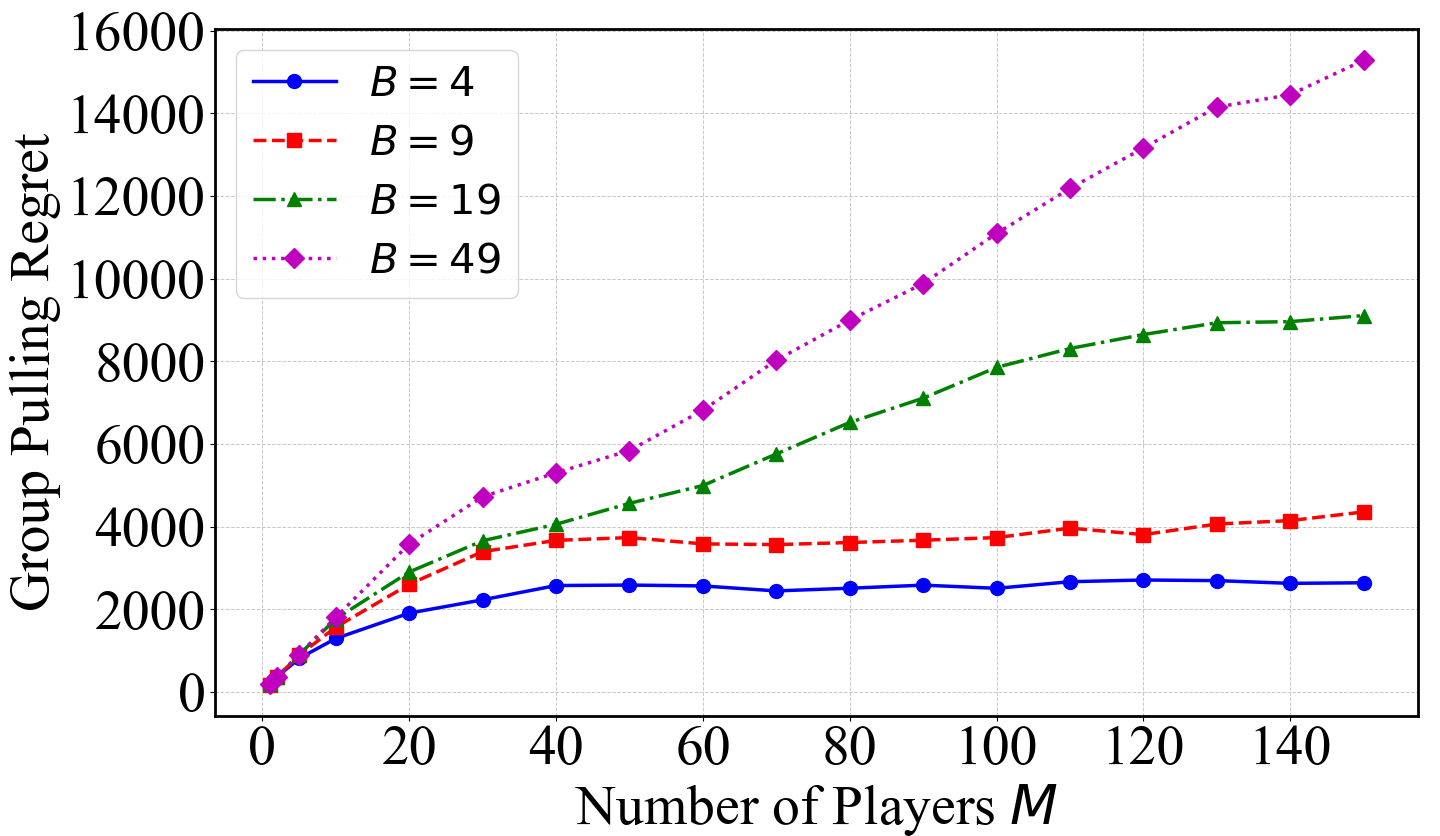}
        \caption{Group pulling regret vs $M$.}
        \label{fig:buying_cost_vs_M}
    \end{subfigure}
    \hfill
    \begin{subfigure}[t]{0.32\textwidth}
        \centering
        \includegraphics[width=\linewidth]{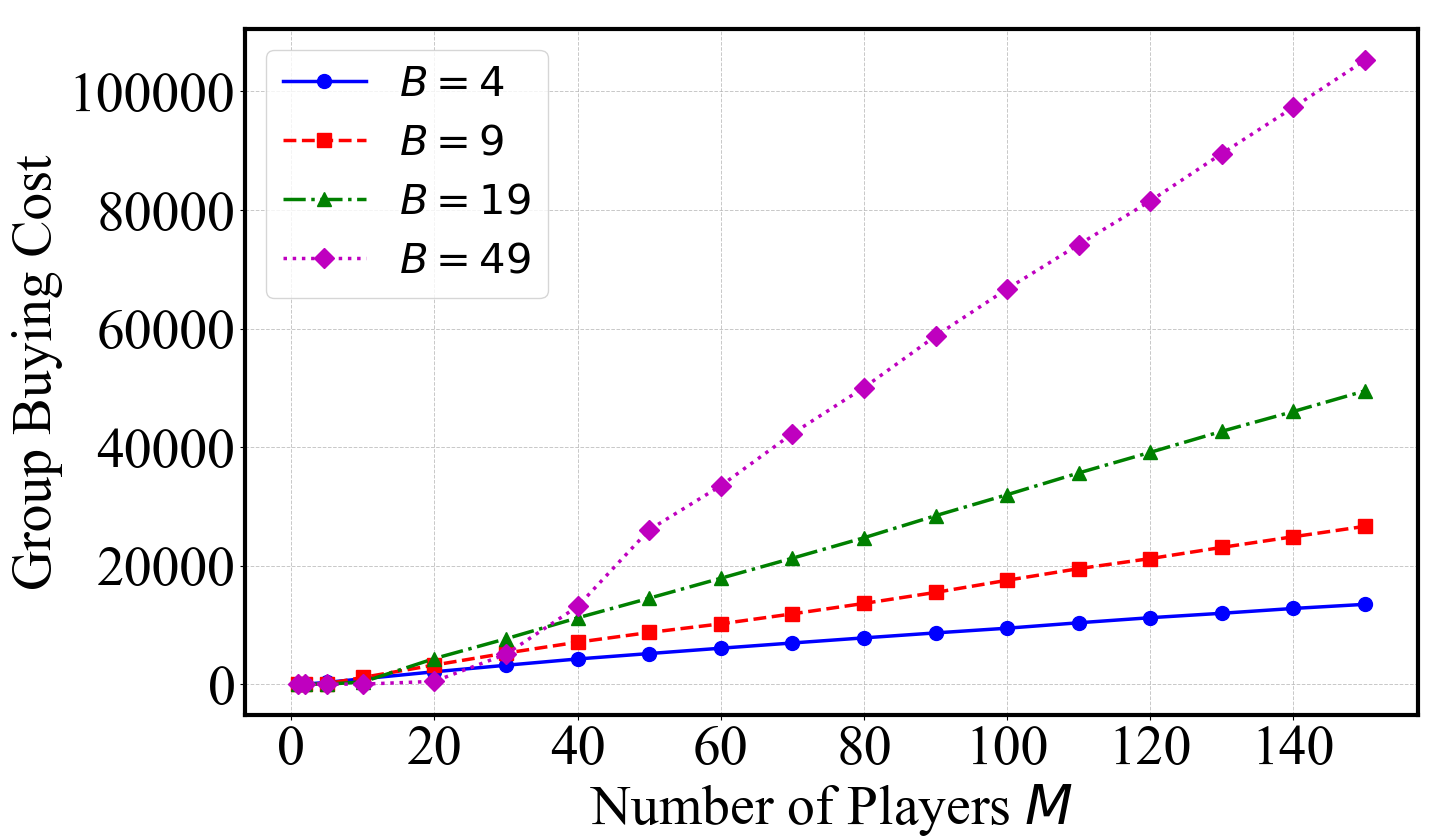}
        \caption{Group buying cost vs $M$.}
        \label{fig:buying_frequency_vs_B}
    \end{subfigure}
    \hfill
    \begin{subfigure}[t]{0.32\textwidth}
        \centering
        \includegraphics[width=\linewidth]{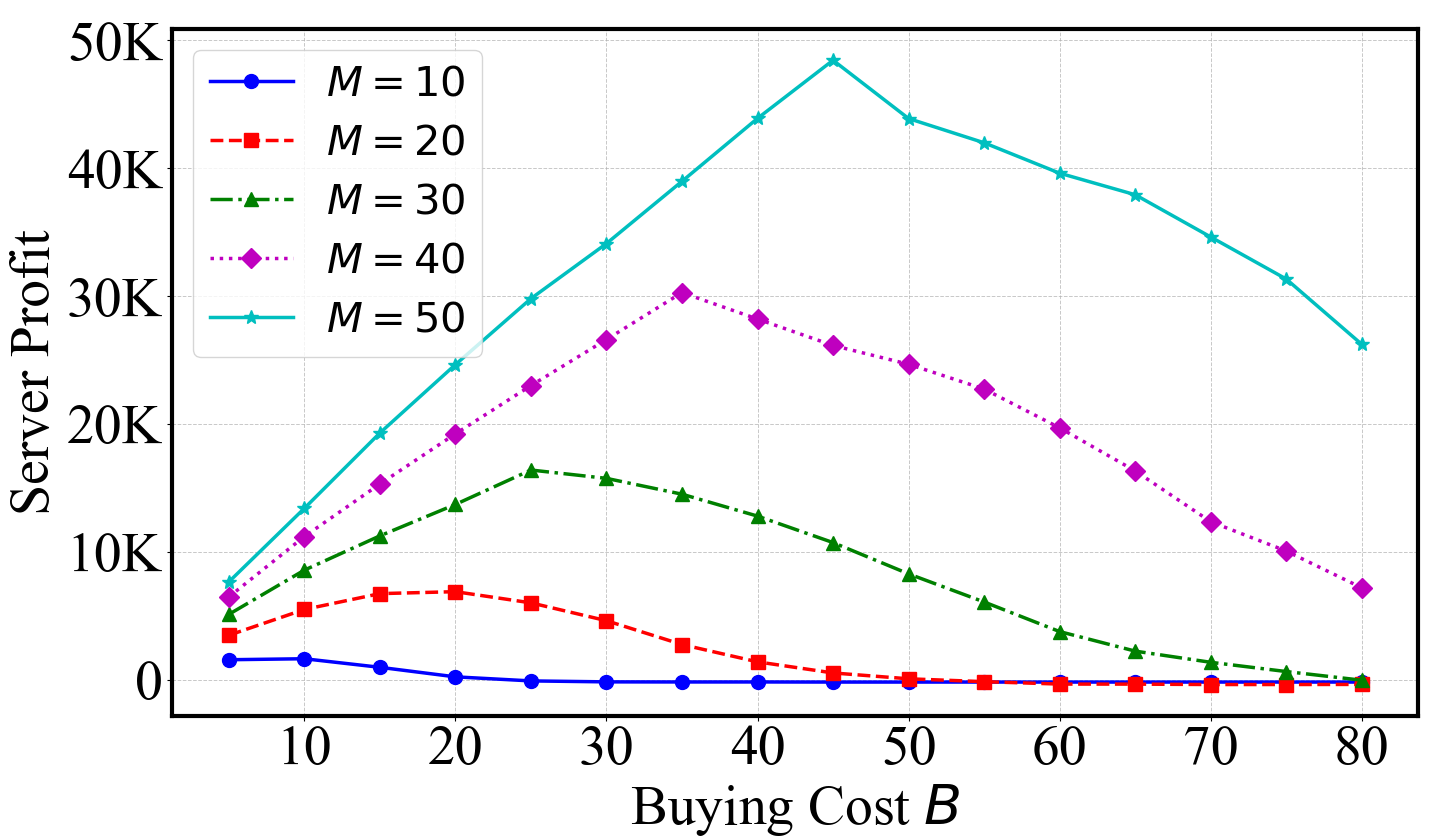}
        \caption{Server profit vs buying cost $B$ for different numbers of players $M$.}
        \label{fig:Server Profit}
    \end{subfigure}

    \caption{Numerical results for the proposed Buying-UCB algorithm under different collaboration and pricing settings.}
    \label{fig:combined_results}
\end{figure*}

In this section, we evaluate the performance of the proposed Buying-UCB algorithm through numerical simulations. We consider a stochastic multi-armed bandit environment with $N=20$ arms and time horizon $T=10000$. The reward of each arm follows a Bernoulli distribution with mean $\mu_k$, where the arm means are generated uniformly at random from the interval $[0.1,0.9]$.  Each experiment is repeated for $50$ independent runs and results are averaged.
\paragraph{Experiment 1: Total Individual Regret vs Time for Different Buying Costs}

Fig.~\ref{fig:regret_vs_time} shows the cumulative individual regret of Algorithm~\ref{alg:buying_ucb} for different buying costs $B$ with $M=50$ collaborating players. We compare Buying-UCB with Independent-UCB, where players learn without collaboration.

Buying-UCB consistently outperforms Independent-UCB for small values of $B$. When buying information is inexpensive, players frequently purchase observations from the server, increasing their effective sample size and accelerating learning. As $B$ increases, buying becomes less frequent, forcing players to rely more on local exploration. Consequently, the regret increases and gradually approaches that of Independent-UCB. These results demonstrate the trade-off between collaboration gain and information acquisition cost.
\paragraph{Experiment 2: Individual Buying Cost vs Number of Players}
Fig.~\ref{fig:regret_vs_M} illustrates the individual buying cost as the number of players $M$ increases for different buying costs $B$. For all values of $B$, the individual buying cost initially increases with $M$ because additional collaborating players generate more observations, making the buying condition more likely to be satisfied. As $M$ increases further, the buying cost gradually saturates since each player has already acquired sufficient information and the frequency of additional purchases decreases.

Moreover, larger buying costs lead to higher individual buying costs even though players purchase information less frequently, because each buying event incurs a higher cost. In contrast, smaller values of $B$ result in substantially lower individual buying costs. These results demonstrate the trade-off between collaboration and information acquisition cost and highlight the influence of both the number of players and the buying cost on the overall expenditure.
\paragraph{Experiment 3: Individual Pulling Regret vs Number of Players}

Fig.~\ref{fig:pull_regret_vs_M} shows the individual pulling regret for different buying costs $B$ as the number of players $M$ increases. This experiment excludes the buying cost and measures only the regret due to pulling sub-optimal arms.

The individual pulling regret decreases as $M$ increases because collaboration provides each player with more effective observations, improving reward estimation. The reduction is most significant for small values of $B$, where information is purchased frequently and collaboration is strongest. As $B$ increases, buying becomes less frequent, reducing the benefit of collaboration and slowing the decrease in regret. Nevertheless, increasing the number of players consistently improves learning performance. 

\paragraph{Experiment 4: Group Pulling Regret vs Number of Players}

Fig.~\ref{fig:buying_cost_vs_M} shows the group pulling regret of Buying-UCB for different buying costs $B$ as the number of players $M$ increases. Since this metric sums the pulling regret over all players, it naturally increases with $M$.

The dependence on $M$, however, varies with the buying cost. For small $B$, the group pulling regret is nearly independent of $M$, consistent with Theorem~\ref{thm:group_pulling_regret}, where the dominant term is proportional to $\min\{M,B+1\}$. Once $M>B+1$, this factor becomes the constant $B+1$. In contrast, for larger $B$, collaboration becomes less frequent and $\min\{M,B+1\}=M$, causing the dominant regret term to grow approximately linearly with $M$. These observations closely match the theoretical bound.

\paragraph{Experiment 5: Group Buying Cost vs Number of Players}

Fig.~\ref{fig:buying_frequency_vs_B} shows the total buying cost incurred by all players as the number of players $M$ increases. Unlike the individual buying cost, this metric captures the overall information acquisition cost of the collaborative system.

The group buying cost increases almost linearly with $M$ since each player purchases information independently whenever the buying condition is satisfied. Larger values of $B$ produce substantially higher buying costs because each purchase becomes more expensive, even though buying occurs less frequently. These observations agree with Theorem~\ref{thm:buying_regret}, which predicts linear growth with $M$ and logarithmic growth with the time horizon.

\paragraph{Experiment 6: Server Profit vs Buying Cost}

Fig.~\ref{fig:Server Profit} shows the net profit of the central server as a function of the buying cost $B$ for different numbers of players $M$. In this experiment, the server pays a fixed sharing reward of $S=0.01$ to each player at every time step. The server profit is computed according to \eqref{eqn:Server Revenue}.

From the figure, we observe that the server profit is not a monotonic function of the buying cost. For small values of $B$, the buying frequency is high, but the revenue obtained from each purchase is small, resulting in relatively low profit. As $B$ increases, the revenue per purchase increases while the buying frequency remains sufficiently high, causing the server profit to improve. However, beyond a certain point, a further increase in $B$ discourages players from purchasing information, leading to a sharp reduction in the number of buys and consequently a decrease in the server's profit.

The optimal buying cost depends on the number of players. As the number of players increases, the server can charge a higher buying price while still maintaining a large number of purchases. In particular, the maximum profit is achieved around $B\approx10$ for $M=20$, $B\approx25$ for $M=30$, $B\approx35$ for $M=40$, and $B\approx45$ for $M=50$. This demonstrates that the server should adapt the buying price according to the network size rather than selecting a fixed price for all deployments.

Overall, the results illustrate an important pricing trade-off: charging too little reduces the revenue obtained from each purchase, whereas charging too much significantly decreases the buying frequency. Therefore, there exists an intermediate buying cost that maximizes the server's overall profit.


\section{Conclusion and Future Work}
\label{sec:conclusion}

In this paper, we studied an incentive-aware federated multi-player multi-armed bandit framework where players earn rewards by sharing information with a central server and incur a cost for purchasing information. We proposed the Buying-UCB algorithm, which adaptively balances exploration and collaboration based on the buying cost. We derived upper bounds on the group pulling regret and buying cost, showing that the algorithm smoothly interpolates between fully collaborative federated learning and independent learning. Numerical results validate the theoretical analysis and illustrate the trade-off between collaboration gain and information acquisition cost.

As future work, we plan to extend the framework to heterogeneous federated bandits, where players solve different MAB instances, making collaboration and information aggregation significantly more challenging. Also we plan to derive theoretical upper bound on individual regret. Another promising direction is to study budget-constrained settings in which pulling an arm also incurs a cost, requiring new algorithms and regret analysis.

\bibliographystyle{IEEEtran}
\bibliography{IEEEabrv,References}

\begin{thebibliography}{10}
\providecommand{\url}[1]{#1}
\csname url@samestyle\endcsname
\providecommand{\newblock}{\relax}
\providecommand{\bibinfo}[2]{#2}
\providecommand{\BIBentrySTDinterwordspacing}{\spaceskip=0pt\relax}
\providecommand{\BIBentryALTinterwordstretchfactor}{4}
\providecommand{\BIBentryALTinterwordspacing}{\spaceskip=\fontdimen2\font plus
\BIBentryALTinterwordstretchfactor\fontdimen3\font minus \fontdimen4\font\relax}
\providecommand{\BIBforeignlanguage}[2]{{%
\expandafter\ifx\csname l@#1\endcsname\relax
\typeout{** WARNING: IEEEtran.bst: No hyphenation pattern has been}%
\typeout{** loaded for the language `#1'. Using the pattern for}%
\typeout{** the default language instead.}%
\else
\language=\csname l@#1\endcsname
\fi
#2}}
\providecommand{\BIBdecl}{\relax}
\BIBdecl

\bibitem{chen2013combinatorial}
W.~Chen, Y.~Wang, and Y.~Yuan, ``Combinatorial multi-armed bandit: General framework and applications,'' in \emph{International conference on machine learning}.\hskip 1em plus 0.5em minus 0.4em\relax PMLR, 2013, pp. 151--159.

\bibitem{gupta2018low}
H.~Gupta, A.~Eryilmaz, and R.~Srikant, ``Low-complexity, low-regret link rate selection in rapidly-varying wireless channels,'' in \emph{IEEE INFOCOM 2018-IEEE Conference on Computer Communications}.\hskip 1em plus 0.5em minus 0.4em\relax IEEE, 2018, pp. 540--548.

\bibitem{tang2015personalized}
L.~Tang, Y.~Jiang, L.~Li, C.~Zeng, and T.~Li, ``Personalized recommendation via parameter-free contextual bandits,'' in \emph{Proceedings of the 38th international ACM SIGIR conference on research and development in information retrieval}, 2015, pp. 323--332.

\bibitem{lattimore2020bandit}
T.~Lattimore and C.~Szepesv{\'a}ri, \emph{Bandit algorithms}.\hskip 1em plus 0.5em minus 0.4em\relax Cambridge University Press, 2020.

\bibitem{auer2002finite}
P.~Auer, N.~Cesa-Bianchi, and P.~Fischer, ``Finite-time analysis of the multiarmed bandit problem,'' \emph{Machine learning}, vol.~47, no.~2, pp. 235--256, 2002.

\bibitem{bubeck2012regret}
S.~Bubeck and N.~Cesa-Bianchi, ``Regret analysis of stochastic and nonstochastic multi-armed bandit problems,'' \emph{arXiv preprint arXiv:1204.5721}, 2012.

\bibitem{agrawal2012analysis}
S.~Agrawal and N.~Goyal, ``Analysis of thompson sampling for the multi-armed bandit problem,'' in \emph{Conference on learning theory}.\hskip 1em plus 0.5em minus 0.4em\relax JMLR Workshop and Conference Proceedings, 2012, pp. 39--1.

\bibitem{6483307}
D.~Kalathil, N.~Nayyar, and R.~Jain, ``Multi-player multi-armed bandits: Decentralized learning with iid rewards,'' in \emph{2012 50th Annual Allerton Conference on Communication, Control, and Computing (Allerton)}, 2012, pp. 853--860.

\bibitem{besson2018multi}
L.~Besson and E.~Kaufmann, ``Multi-player bandits revisited,'' in \emph{Algorithmic Learning Theory}.\hskip 1em plus 0.5em minus 0.4em\relax PMLR, 2018, pp. 56--92.

\bibitem{11414221}
L.~Yang, X.~Wang, H.~Chen, M.~H. Hajiesmaili, L.~Zhang, J.~C.~S. Lui, and D.~Towsley, ``Cooperative bandit algorithms with optimal regret and communication costs,'' \emph{IEEE Transactions on Networking}, vol.~34, pp. 3609--3624, 2026.

\bibitem{karpov2024communication}
N.~Karpov and Q.~Zhang, ``Communication-efficient collaborative regret minimization in multi-armed bandits,'' in \emph{Proceedings of the AAAI Conference on Artificial Intelligence}, vol.~38, no.~12, 2024, pp. 13\,076--13\,084.

\bibitem{wang2022achieving}
X.~Wang, L.~Yang, Y.-Z.~J. Chen, X.~Liu, M.~Hajiesmaili, D.~Towsley, and J.~C. Lui, ``Achieving near-optimal individual regret \& low communications in multi-agent bandits,'' in \emph{The Eleventh International Conference on Learning Representations}, 2022.

\bibitem{zhu2025decentralized}
J.~Zhu, E.~Mulle, C.~S. Smith, A.~Koppel, and J.~Liu, ``Decentralized upper confidence bound algorithms for homogeneous multi-agent multi-armed bandits,'' \emph{IEEE Transactions on Automatic Control}, vol.~70, no.~7, pp. 4360--4375, 2025.

\bibitem{madhushani2021call}
U.~Madhushani and N.~Leonard, ``When to call your neighbor? strategic communication in cooperative stochastic bandits,'' \emph{arXiv preprint arXiv:2110.04396}, 2021.

\bibitem{chawla2020gossiping}
R.~Chawla, A.~Sankararaman, A.~Ganesh, and S.~Shakkottai, ``The gossiping insert-eliminate algorithm for multi-agent bandits,'' in \emph{International conference on artificial intelligence and statistics}.\hskip 1em plus 0.5em minus 0.4em\relax PMLR, 2020, pp. 3471--3481.

\bibitem{wang2020optimal}
P.-A. Wang, A.~Proutiere, K.~Ariu, Y.~Jedra, and A.~Russo, ``Optimal algorithms for multiplayer multi-armed bandits,'' in \emph{International Conference on Artificial Intelligence and Statistics}.\hskip 1em plus 0.5em minus 0.4em\relax PMLR, 2020, pp. 4120--4129.

\bibitem{bistritz2020cooperative}
I.~Bistritz and N.~Bambos, ``Cooperative multi-player bandit optimization,'' \emph{Advances in Neural Information Processing Systems}, vol.~33, pp. 2016--2027, 2020.

\bibitem{chakraborty2017coordinated}
M.~Chakraborty, K.~Y.~P. Chua, S.~Das, and B.~Juba, ``Coordinated versus decentralized exploration in multi-agent multi-armed bandits.'' in \emph{IJCAI}, 2017, pp. 164--170.

\bibitem{chen2023demand}
Y.-Z.~J. Chen, L.~Yang, X.~Wang, X.~Liu, M.~Hajiesmaili, J.~C. Lui, and D.~Towsley, ``On-demand communication for asynchronous multi-agent bandits,'' in \emph{International Conference on Artificial Intelligence and Statistics}.\hskip 1em plus 0.5em minus 0.4em\relax PMLR, 2023, pp. 3903--3930.

\bibitem{liu2026distributed}
J.~Liu, H.~Qiu, L.~Yang, and M.~Xu, ``Distributed multi-agent bandits over erd{\H{o}}s-r{\'e}nyi random networks,'' \emph{Advances in Neural Information Processing Systems}, vol.~38, pp. 159\,769--159\,798, 2026.

\bibitem{karimireddy2022mechanisms}
S.~P. Karimireddy, W.~Guo, and M.~I. Jordan, ``Mechanisms that incentivize data sharing in federated learning,'' \emph{arXiv preprint arXiv:2207.04557}, 2022.

\bibitem{hoeffding1963probability}
W.~Hoeffding, ``Probability inequalities for sums of bounded random variables,'' \emph{Journal of the American statistical association}, vol.~58, no. 301, pp. 13--30, 1963.

\end{thebibliography}
\vspace{12pt}

\appendices

\renewcommand{\thesubsection}{\thesection.\arabic{subsection}}
\makeatletter
\renewcommand{\@seccntformat}[1]{%
  \csname the#1\endcsname\quad
}
\makeatother

\section{Regret when \texorpdfstring{$B=0$}{B = 0}}
In this section, we consider the special case $B=0$, corresponding to cost-free information sharing. This case illustrates the best possible collaboration achieved by Buying-UCB and verifies that the proposed framework recovers the expected $1/M$ improvement in individual pulling regret. Since information is exchanged without restriction, all players maintain identical statistics throughout the learning process.
\begin{theorem}
\label{thm:pulling_regret_B0}
For $B=0$, the expected pulling regret of Algorithm 1 satisfies
\begin{align}
\mathbb{E}[R_p^{\mathrm{pull}}(T)]
\le
\sum_{k:\Delta_k>0}
\Delta_k
\left(
\frac{1}{M}
\left(
\frac{8\log T}{\Delta_k^2}
+1
\right)
+\frac{2}{3}
\right).
\end{align}
\end{theorem}

\begin{proof}

Fix a sub-optimal arm $k$ and define
\begin{equation}
\label{eq:l_B0}
l
=
\left\lceil
\frac{8\log T}{\Delta_k^2}
\right\rceil.
\end{equation}

Since $B=0$, the buying condition in (18)
is always satisfied because

$$\sum_{q \neq p} n_{q,k}(t)
-
\sum_{q \neq p} n_{q,k}(\tau_{p,k}^{t-1})
\geq
0.$$

Therefore, at every time instant $t$, each player buys all
newly available observations from the server. As a result, all
players share the same information set and the effective number
of samples available to player $p$ from arm $k$ is $$N_{p,k}^{+}(t)=Mn_{p,k}(t).$$

We first upper bound the number of pulls of arm $k$ by player $p$:
\begin{align}
n_{p,k}(T+1)
&=
\sum_{t=1}^{T}
\mathbf{1}\{A_p(t)=k\}
\nonumber\\
&=
\sum_{t=1}^{T}
\mathbf{1}
\{
A_p(t)=k,\,
N_{p,k}^{+}(t)\le l
\}
\nonumber\\
&\quad+
\sum_{t=1}^{T}
\mathbf{1}
\{
A_p(t)=k,\,
N_{p,k}^{+}(t)>l
\}.
\label{eq:B0_decomposition}
\end{align}

Since $N_{p,k}^{+}(t)=Mn_{p,k}(t)$, the condition
$N_{p,k}^{+}(t)\le l$ implies

$$n_{p,k}(t)\le \frac{l}{M}.$$

Therefore,
\begin{align}
n_{p,k}(T+1)
&\le
\left\lceil
\frac{l}{M}
\right\rceil
+
\sum_{t=\left\lceil \frac{l}{M}\right\rceil+1}^{T}
\mathbf{1}
\{
A_p(t)=k,\,
N_{p,k}^{+}(t)>l
\}
\nonumber\\
&\le
\left\lceil
\frac{l}{M}
\right\rceil
+
\sum_{t=\left\lceil \frac{l}{M}\right\rceil+1}^{T}
\mathbf{1}
\{
\mathrm{UCB}_{p,k}(t)
\ge
\mathrm{UCB}_{p,*}(t),\nonumber \\
&\hspace{4cm} N_{p,k}^{+}(t)>l
\}.
\label{eq:B0_pull_count}
\end{align}

Now, if player $p$ selects a sub-optimal arm $k$ at time $t$, then
$\mathrm{UCB}_{p,k}(t)
\ge
\mathrm{UCB}_{p,*}(t).$ Consequently, at least one of the following events must occur:
\begin{subequations}
\begin{align}
\hat{\mu}_{p,*}(t)
&\le
\mu_*
-
C_{p,*}(t),
\\
\hat{\mu}_{p,k}(t)
&\ge
\mu_k
+
C_{p,k}(t),
\\
\Delta_k
&\le
2C_{p,k}(t).
\end{align}
\end{subequations}

Using Hoeffding's inequality ~\cite[Theorem 2]{hoeffding1963probability},
\begin{subequations}
\begin{align}
\mathbb{P}
\left(
\hat{\mu}_{p,k}(t)-\mu_k
\ge
C_{p,k}(t)
\right)
&\le
\frac{1}{t^4},
\\
\mathbb{P}
\left(
\hat{\mu}_{p,*}(t)-\mu_*
\le
-C_{p,*}(t)
\right)
&\le
\frac{1}{t^4}.
\end{align}
\end{subequations}

Furthermore, if $N_{p,k}^{+}(t)
>
\frac{8\log T}{\Delta_k^2},$ then
$2C_{p,k}(t)
<
\Delta_k,$ and therefore the third event cannot occur. Thus, choosing $l$ as in
\eqref{eq:l_B0}, we obtain

\begin{align}
\mathbb{E}[n_{p,k}(T+1)]
&\le
\left\lceil
\frac{l}{M}
\right\rceil
+
\sum_{t=\left\lceil \frac{l}{M}\right\rceil+1}^{T}
\mathbb{P}
[\text{Event 1 or Event 2}]
\nonumber\\
&\le
\left\lceil
\frac{l}{M}
\right\rceil
+
\sum_{t=\left\lceil \frac{l}{M}\right\rceil+1}^{T}
\frac{2}{t^4}
\nonumber\\
&\le
\left\lceil
\frac{l}{M}
\right\rceil
+
2
\int_{\left\lceil \frac{l}{M}\right\rceil}^{\infty}
\frac{1}{t^4}\,dt
\nonumber\\
&=
\left\lceil
\frac{l}{M}
\right\rceil
+
\frac{2}{3}
\left(
\frac{M}{l}
\right)^3.
\end{align}

Since $\left(\frac{M}{l}\right)^3\le 1$, it follows that

$$\mathbb{E}[n_{p,k}(T+1)]
\le
\left\lceil
\frac{l}{M}
\right\rceil
+
\frac{2}{3}.$$

Using

$$l
=
\left\lceil
\frac{8\log T}{\Delta_k^2}
\right\rceil
\le
\frac{8\log T}{\Delta_k^2}
+1,$$

we obtain

$$\mathbb{E}[n_{p,k}(T+1)]
\le
\frac{1}{M}
\left(
\frac{8\log T}{\Delta_k^2}
+1
\right)
+
\frac{2}{3}.$$

Finally, substituting the above bound into

$$R_p^{\mathrm{pull}}(T)
=
\sum_{k:\Delta_k>0}
\Delta_k\,n_{p,k}(T+1)$$

yields

$$\mathbb{E}[R_p^{\mathrm{pull}}(T)]
\le
\sum_{k:\Delta_k>0}
\Delta_k
\left(
\frac{1}{M}
\left(
\frac{8\log T}{\Delta_k^2}
+1
\right)
+\frac{2}{3}
\right),$$

which proves the theorem.
\end{proof}

\begin{theorem}
\label{thm:group_regret_B0}

Suppose $B=0$. Then the expected group pulling regret satisfies
\[
\mathbb{E}\!\left[
R_{\mathrm{group}}^{\mathrm{pull}}(T)
\right]
=
O(\log T).
\]

\end{theorem}

\begin{proof}

When $B=0$, the buying condition in (18)
is always satisfied. Hence, after every round, each player purchases all
new observations from the server. Consequently, all players maintain
identical statistics throughout the learning process, i.e., $$N_{p,k}^{+}(t)=Mn_{p,k}(t),
\qquad
\forall\,p,\;k,\;t.
$$
Therefore, every player incurs the same expected pulling regret. Then
\[
\mathbb{E}
\!\left[
R_{\mathrm{group}}^{\mathrm{pull}}(T)
\right]
=
\sum_{p=1}^{M}
\mathbb{E}
\!\left[
R_p^{\mathrm{pull}}(T)
\right]
=
M\,
\mathbb{E}
\!\left[
R_p^{\mathrm{pull}}(T)
\right].
\]

From Theorem~\ref{thm:pulling_regret_B0},
\[
\mathbb{E}
\!\left[
R_p^{\mathrm{pull}}(T)
\right]
=
O\!\left(\frac{\log T}{M}\right).
\]

Hence,

\[
\mathbb{E}
\!\left[
R_{\mathrm{group}}^{\mathrm{pull}}(T)
\right]
=
M\cdot
O\!\left(\frac{\log T}{M}\right)
=
O(\log T),
\]

which completes the proof.

\end{proof}

\section{Preliminaries}

The definitions and lemmas presented in this section are adapted from [12]. They are presented here for completeness, since our analysis differs from [12] in that each player maintains a distinct confidence set as a result of asynchronous information buying. These results therefore provide the notation and probabilistic guarantees used throughout the subsequent regret analysis.
\begin{definition}[Type-I and Type-II Decisions] \label{def:type_decision}
A player $p$ is said to make a \emph{Type-I decision} at time $t$
if, for every arm $k\in[N]$, the true mean reward $\mu_k$
lies inside the confidence interval
$\bigl(\mathrm{tLCB}_{p,k}(\alpha,t),
\mathrm{tUCB}_{p,k}(\alpha,t)\bigr)$, i.e.,
\[
\mu_k\in
\left(
\mathrm{tLCB}_{p,k}(\alpha,t),
\mathrm{tUCB}_{p,k}(\alpha,t)
\right),
\qquad
\forall k\in[N].
\]

Otherwise, player $p$ is said to make a
\emph{Type-II decision}.
\end{definition}

\begin{lemma}
\label{lem:confidence_interval}

For any player $p$, arm $k$, and time $t$, the probability that
the true mean reward $\mu_k$ lies inside the confidence interval
\[
\left(
\mathrm{tLCB}_{p,k}(\alpha,t),
\mathrm{tUCB}_{p,k}(\alpha,t)
\right)
\]
is at least $1-2t^{\,-2\alpha^2}$.
\end{lemma}

\begin{proof}

We first bound the probability that the true mean reward is greater than
the upper confidence bound.

\begin{align*}
&\Pr\!\left(
\mu_k
\ge
\mathrm{tUCB}_{p,k}(\alpha,t)
\right)
\\
&=
\Pr
\left(
\mu_k-\hat{\mu}_{p,k}(t)
\ge
\alpha
\sqrt{
\frac{\log t}
{N_{p,k}^{+}(t)}
}
\right)
\\
&\stackrel{(a)}{\le} t^{-2\alpha^2},
\end{align*}

where inequality $(a)$ follows from Hoeffding's inequality~\cite[Theorem 2]{hoeffding1963probability}.

Similarly,

\begin{align*}
\Pr
\left(
\mu_k
\le
\mathrm{tLCB}_{p,k}(\alpha,t)
\right)
\le
t^{-2\alpha^2}.
\end{align*}

Therefore, by excluding the above two events,

\begin{align*}
&\Pr
\left(
\mu_k
\notin
\left(
\mathrm{tLCB}_{p,k}(\alpha,t),
\mathrm{tUCB}_{p,k}(\alpha,t)
\right)
\right)
\\
&\le
2t^{-2\alpha^2}.
\end{align*}

Equivalently,

\begin{align*}
\Pr
\left(
\mu_k
\in
\left(
\mathrm{tLCB}_{p,k}(\alpha,t),
\mathrm{tUCB}_{p,k}(\alpha,t)
\right)
\right)
\ge
1-2t^{-2\alpha^2},
\end{align*}

which completes the proof.

\end{proof}

\begin{lemma}
\label{lem:type2}

Suppose $\alpha>1/\sqrt{2}$.
Then the expected total number of Type-II decisions
made by all players is finite. Specifically, it is less than $
\frac{4MN\alpha^2}
{2\alpha^2-1}$.

\end{lemma}

\begin{proof}

A Type-II decision occurs whenever there exists at least one arm whose
true mean lies outside its confidence interval. Therefore, the expected
total number of Type-II decisions made by player $p$ can be bounded as

\begin{align*}
&
\sum_{t=1}^{T}
\Pr(\text{player $p$ makes Type-II decision at time }t)
\\
&=
\sum_{t=1}^{T}
\Pr
\left(
\exists\,k\in[N]:
\mu_k
\notin
\left(
\mathrm{tLCB}_{p,k}(\alpha,t),
\mathrm{tUCB}_{p,k}(\alpha,t)
\right)
\right)
\\
&\le
\sum_{t=1}^{T}
\sum_{k=1}^{N}
\Pr
\left(
\mu_k
\notin
\left(
\mathrm{tLCB}_{p,k}(\alpha,t),
\mathrm{tUCB}_{p,k}(\alpha,t)
\right)
\right)
\\
&\stackrel{(a)}{\le}
2N
\sum_{t=1}^{T}
t^{\,-2\alpha^2}
\\
&\le
2N
\left(
1+
\int_{t=1}^{T}
t^{\,-2\alpha^2}
dt
\right)
\\
&\stackrel{(b)}{\le}
\frac{4N\alpha^2}
{2\alpha^2-1},
\end{align*}

where inequality $(a)$ follows from Lemma~\ref{lem:confidence_interval},
and inequality $(b)$ follows from $\alpha>1/\sqrt{2}$.

Summing over all $M$ players yields

\begin{align*}
\sum_{p=1}^{M}
\sum_{t=1}^{T}
\Pr(\text{player $p$ makes a Type-II decision})
\le
\frac{4N\alpha^2}
{2\alpha^2-1}.
\end{align*}
which completes the proof.
\end{proof}

\begin{lemma}
\label{lem:type1_confidence_set}
Suppose player $p$ makes a Type-I decision at time $t$. Then every
sub-optimal arm belongs to the confidence set of player $p$, i.e.,
$k\in\mathcal{C}_{p,t}(\alpha),
\forall\,k\neq k^\star.$ 
\end{lemma}

\begin{proof}

Suppose player $p$ makes a Type-I decision at time $t$.
By Definition~\ref{def:type_decision}, every arm satisfies

\begin{align}
\mu_k
\in
\left(
\mathrm{tLCB}_{p,k}(\alpha,t),
\mathrm{tUCB}_{p,k}(\alpha,t)
\right),
\qquad
\forall k.
\label{eq:type1_all}
\end{align}

Let $k$ be any sub-optimal arm. Since $k^\star$ is the optimal arm,
$\mu_k<\mu_{k^\star}.$
Using \eqref{eq:type1_all},

\begin{align*}
\mathrm{tLCB}_{p,k}(\alpha,t)
&<
\mu_k
\\
&<
\mu_{k^\star}
\\
&<
\mathrm{tUCB}_{p,k^\star}(\alpha,t)
\\
&\le
\max_{k'\neq k}
\mathrm{tUCB}_{p,k'}(\alpha,t),
\end{align*}

where the last inequality follows because
$k^\star\neq k$.

Therefore,

\[
\mathrm{tLCB}_{p,k}(\alpha,t)
<
\max_{k'\neq k}
\mathrm{tUCB}_{p,k'}(\alpha,t),
\]

which, by the definition of the confidence set (16), implies

\[
k\in\mathcal{C}_{p,t}(\alpha).
\]

Since the choice of the sub-optimal arm $k$ was arbitrary, every
sub-optimal arm belongs to the confidence set.

This completes the proof.

\end{proof}

\section{Proof of Lemma 1 (Main Paper)}
\label{sec:proof_lemma1}

\noindent\textbf{Lemma 1.}
The expected group pulling regret satisfies
\begin{align}
\mathbb{E}
\!\left[
R_{\mathrm{group}}^{\mathrm{pull}}(T)
\right]
=
\sum_{k:\Delta_k>0}
\Delta_k\,
\mathbb{E}
\!\left[
n_k^{\mathrm{group}}(T+1)
\right].
\label{eq:group_pulling_gap}
\end{align}

\begin{proof}
Using the definition of the sub-optimality gap
$\Delta_k=\mu_*-\mu_k$, the expected group pulling regret
can be written as
\begin{align}
\mathbb{E}
\!\left[
R_{\mathrm{group}}^{\mathrm{pull}}(T)
\right]
&=
MT\mu_*
-
\sum_{p=1}^{M}
\sum_{t=1}^{T}
\mathbb{E}
\!\left[
r_{p,A_p(t)}(t)
\right]
\nonumber\\
&=
MT\mu_*
-
\sum_{p=1}^{M}
\sum_{k=1}^{N}
\mu_k\,
\mathbb{E}
\!\left[
n_{p,k}(T+1)
\right]
\nonumber\\
&=
\sum_{p=1}^{M}
\sum_{k=1}^{N}
\mu_*\,
\mathbb{E}
\!\left[
n_{p,k}(T+1)
\right] \nonumber \\
& \qquad -
\sum_{p=1}^{M}
\sum_{k=1}^{N}
\mu_k\,
\mathbb{E}
\!\left[
n_{p,k}(T+1)
\right]
\nonumber\\
&=
\sum_{p=1}^{M}
\sum_{k=1}^{N}
(\mu_*-\mu_k)\,
\mathbb{E}
\!\left[
n_{p,k}(T+1)
\right]
\nonumber\\
&\stackrel{(a)}{=}
\sum_{p=1}^{M}
\sum_{k:\Delta_k>0}
\Delta_k\,
\mathbb{E}
\!\left[
n_{p,k}(T+1)
\right]
\nonumber\\
&=
\sum_{k:\Delta_k>0}
\Delta_k
\sum_{p=1}^{M}
\mathbb{E}
\!\left[
n_{p,k}(T+1)
\right]
\nonumber\\
&\stackrel{(b)}{=}
\sum_{k:\Delta_k>0}
\Delta_k\,
\mathbb{E}
\!\left[
n_k^{\mathrm{group}}(T+1)
\right].
\end{align}

where $(a)$ follows from the definition of the sub-optimality gap
$\Delta_k=\mu_*-\mu_k$, and
$(b)$ follows from the definition of
$n_k^{\mathrm{group}}(T+1)$.

\end{proof}

\section{Proof of Theorem 1 (Main Paper)}
\label{sec:proof_theorem1}

\noindent\textbf{Theorem 1.}
For $\alpha > 1/\sqrt{2}$, the expected group pulling regret of
Algorithm~1 satisfies
\begin{align}
\mathbb{E}
\!\left[
R_{\mathrm{group}}^{\mathrm{pull}}(T)
\right]
&\le
\sum_{k:\Delta_k>0}
\Delta_k
\Bigg[
\min\{M,B+1\}
\left(
\frac{8\log T}{\Delta_k^2}
+2
\right)
\nonumber\\
&\qquad
+\frac{4MN\alpha^2}{2\alpha^2-1}
+\frac{2M}{3}
\Bigg].
\label{eq:group_pulling_regret_bound}
\end{align}

\begin{proof}
Fix a sub-optimal arm $k$,
let
\begin{equation}
\label{eq:l_group}
l
=
\left\lceil
\frac{8\log T}{\Delta_k^2}
\right\rceil.
\end{equation}

\begin{align}
\lefteqn{n_k^{\mathrm{group}}(T+1)} \nonumber \\
&=
\sum_{t=1}^{T}
\sum_{p=1}^{M}
\mathbf{1}\{A_p(t)=k\}
\nonumber\\
&=
\sum_{t=1}^{T}
\sum_{p=1}^{M}
\mathbf{1}
\left\{
A_p(t)=k,\,
N_{p,k}^{+}(t)\le l
\right\}
\nonumber\\
&\quad+
\sum_{t=1}^{T}
\sum_{p=1}^{M}
\mathbf{1}
\left\{
A_p(t)=k,\,
N_{p,k}^{+}(t)>l
\right\}
\nonumber\\
&=
\sum_{t=1}^{T}
\sum_{p=1}^{M}
\mathbf{1}
\left\{
A_p(t)=k,\,
N_{p,k}^{+}(t)\le l,\,
\text{Type-I}
\right\}
\nonumber\\
&\quad+
\sum_{t=1}^{T}
\sum_{p=1}^{M}
\mathbf{1}
\left\{
A_p(t)=k,\, N_{p,k}^{+}(t)\le l,\,
\text{Type-II}
\right\}
\nonumber\\
&\quad+
\sum_{t=1}^{T}
\sum_{p=1}^{M}
\mathbf{1}
\left\{
A_p(t)=k,\,
N_{p,k}^{+}(t)>l
\right\}.
\label{eq:nonbuy_split}
\end{align}

We first bound the Type-I term

\begin{align}
&
\sum_{t=1}^{T}
\sum_{p=1}^{M}
\mathbf{1}
\left\{
A_p(t)=k,\,
N_{p,k}^{+}(t)\le l,\,
\text{Type-I}
\right\}.
\label{eq:type1_term}
\end{align}

From Lemma~\ref{lem:type1_confidence_set}, whenever player $p$
makes a Type-I decision, every sub-optimal arm belongs to the
confidence set. Consequently, every sub-optimal arm remains eligible
for information buying. We now analyze the
number of group pulls while
$ N_{p,k}^{+}(t)\le l$.

We analyze two cases.

\medskip
\noindent
\textbf{Case 1: Player $p$ buys information for arm $k$ at time $t$.}

By the update rule of Algorithm 1, whenever
player $p$ buys information for arm $k$, all observations
available in the system regarding arm $k$ are incorporated into
the effective sample count. Hence,

\begin{align}
N_{p,k}^{+}(t)
=
n_k^{\mathrm{group}}(t).
\label{eq:buyer_case}
\end{align}

\medskip
\noindent
\textbf{Case 2: Player $p$ does not buy information for arm $k$ at time $t$.}

Since player $p$ makes a Type-I decision,
Lemma~\ref{lem:type1_confidence_set} implies that the sub-optimal arm
$k$ belongs to the confidence set. Therefore, the player has already
checked the buying condition for arm $k$. As the player does not buy,
the buying condition must not be satisfied. Hence,

\begin{align}
\sum_{q\neq p}
n_{q,k}(t)
-
\sum_{q\neq p}
n_{q,k}
\!\left(
\tau_{p,k}^{t-1}
\right)
<
B\,N_{p,k}(t).
\label{eq:nonbuyer_condition}
\end{align}

Adding and subtracting $n_{p,k}(t)$ on the left-hand side yields

\begin{align}
n_k^{\mathrm{group}}(t)
<
(B+1)N_{p,k}(t)
=
(B+1)N_{p,k}^{+}(t).
\label{eq:nonbuyer_case}
\end{align}

For every time $t$, define the set of players making a Type-I
decision by
\begin{equation}
\mathcal I_t
\triangleq
\left\{
p\in[M]:
\text{player $p$ makes a Type-I decision at time $t$}
\right\}.
\end{equation}

For every $p\in\mathcal I_t$, if player $p$ does not buy
information for arm $k$, then by \eqref{eq:nonbuyer_case},
\begin{align}
n_k^{\mathrm{group}}(t)
<
(B+1)N_{p,k}^{+}(t).
\end{align}
If player $p$ buys information for arm $k$, then by
\eqref{eq:buyer_case},
\begin{align}
N_{p,k}^{+}(t)
=
n_k^{\mathrm{group}}(t),
\end{align}
and hence
\begin{align}
n_k^{\mathrm{group}}(t)
\le
(B+1)N_{p,k}^{+}(t).
\end{align}
Therefore, whenever $\mathcal I_t\neq\varnothing$,
\begin{equation}
\label{eq:type1_group_min_bound}
n_k^{\mathrm{group}}(t)
\le
(B+1)
\min_{p\in\mathcal I_t}
N_{p,k}^{+}(t).
\end{equation}

We now derive two upper bounds for the Type-I term. First, since
$n_{p,k}(t)\le N_{p,k}^{+}(t)$,
\begin{align}
&
\sum_{t=1}^{T}
\sum_{p=1}^{M}
\mathbf{1}
\left\{
A_p(t)=k,\,
N_{p,k}^{+}(t)\le l,\,
\mathrm{Type\!-\!I}
\right\}
\nonumber\\
&\le
\sum_{p=1}^{M}
\sum_{t=1}^{T}
\mathbf{1}
\left\{
A_p(t)=k,\,
n_{p,k}(t)\le l,\,
\mathrm{Type\!-\!I}
\right\}
\nonumber\\
&\le M(l+1).
\label{eq:type1_bound_M}
\end{align}

For the second bound, let $t'$ denote the last time at which
there exists a Type-I player who pulls arm $k$ while its effective
sample count is at most $l$, i.e.,
\begin{equation}
t'
=
\max
\left\{
t\le T:
\exists p\in\mathcal I_t
\text{ such that }
A_p(t)=k,\,
N_{p,k}^{+}(t)\le l
\right\}.
\end{equation}
Thus, every pull counted in the Type-I term occurs no later than
$t'$. Consequently,
\begin{align}
&
\sum_{t=1}^{T}
\sum_{p=1}^{M}
\mathbf{1}
\left\{
A_p(t)=k,\,
N_{p,k}^{+}(t)\le l,\,
\mathrm{Type\!-\!I}
\right\}
\nonumber\\
&\le
n_k^{\mathrm{group}}(t').
\end{align}

By the definition of $t'$, there exists a player
$p'\in\mathcal I_{t'}$ such that
\[
N_{p',k}^{+}(t')\le l.
\]
Using \eqref{eq:type1_group_min_bound}, we then have
\begin{align}
n_k^{\mathrm{group}}(t')
&\le
(B+1)
\min_{q\in\mathcal I_{t'}}
N_{q,k}^{+}(t')
\nonumber\\
&\le
(B+1)N_{p',k}^{+}(t')
\nonumber\\
&\le
(B+1)l.
\end{align}
Therefore,
\begin{equation}
\label{eq:type1_bound_B}
\sum_{t=1}^{T}
\sum_{p=1}^{M}
\mathbf{1}
\left\{
A_p(t)=k,\,
N_{p,k}^{+}(t)\le l,\,
\mathrm{Type\!-\!I}
\right\}
\le
(B+1)l.
\end{equation}
Consequently,
\begin{align}
&
\sum_{t=1}^{T}
\sum_{p=1}^{M}
\mathbf{1}
\left\{
A_p(t)=k,\,
N_{p,k}^{+}(t)\le l,\,
\mathrm{Type\!-\!I}
\right\}
\nonumber\\
&\le
(B+1)l.
\label{eq:type1_bound_B}
\end{align}

Combining \eqref{eq:type1_bound_M} and
\eqref{eq:type1_bound_B}, we obtain
\begin{align}
&
\sum_{t=1}^{T}
\sum_{p=1}^{M}
\mathbf{1}
\left\{
A_p(t)=k,\,
N_{p,k}^{+}(t)\le l,\,
\mathrm{Type\!-\!I}
\right\}
\nonumber\\
&\le
\min\{M(l+1),(B+1)l\}
\nonumber\\
&\le
\min\{M,B+1\}(l+1).
\label{eq:type1_bound}
\end{align}
We now bound the second term in \eqref{eq:nonbuy_split}. Observe that

\begin{align}
&
\sum_{t=1}^{T}
\sum_{p=1}^{M}
\mathbf{1}
\left\{
A_p(t)=k,\,
N_{p,k}^{+}(t)\le l,\,
\text{Type-II}
\right\}
\nonumber\\
&\le
\sum_{t=1}^{T}
\sum_{p=1}^{M}
\mathbf{1}
\left\{
\text{player $p$ makes a Type-II decision at time }t
\right\}.
\label{eq:type2_bound_step}
\end{align}

Taking expectation on both sides and applying
Lemma~\ref{lem:type2},

\begin{align}
&
\mathbb{E}
\!\left[
\sum_{t=1}^{T}
\sum_{p=1}^{M}
\mathbf{1}
\left\{
A_p(t)=k,\,
N_{p,k}^{+}(t)\le l,\,
\text{Type-II}
\right\}
\right]
\nonumber\\
&\le
\sum_{t=1}^{T}
\sum_{p=1}^{M}
\Pr
\left(
\text{player $p$ makes a Type-II decision at time }t
\right)
\nonumber\\
&\le
\frac{4MN\alpha^{2}}
{2\alpha^{2}-1}.
\label{eq:type2_bound}
\end{align}

Now bound the third term in \eqref{eq:nonbuy_split}. Now consider any player $p$.
If player $p$ pulls the sub-optimal arm $k$, then $\mathrm{UCB}_{p,k}(t)
\ge
\mathrm{UCB}_{p,*}(t).$
Therefore at least one of the following events occurs:
\begin{subequations}
\begin{align}
\hat{\mu}_{p,*}(t)
&\le
\mu_*
-
C_{p,*}(t), \label{eqn:Event 1}
\\
\hat{\mu}_{p,k}(t)
&\ge
\mu_k
+
C_{p,k}(t), \label{eqn:Event 2}
\\
\Delta_k
&\le
2C_{p,k}(t). \label{eqn:Event3}
\end{align}
\end{subequations}
We can bound the probabilities of Inequalities \eqref{eqn:Event 1} and \eqref{eqn:Event 2} using Hoeffoding's inequalities~\cite[Theorem 2]{hoeffding1963probability} as follows.
\begin{subequations}
\begin{align}
\mathbb{P}
\Big(
\hat{\mu}_{p,k}(t)-\mu_k
\ge
C_{p,k}(t)
\Big)
&\le
\frac1{t^4}, \label{eqn:Event 1 Probability}
\\
\mathbb{P}
\Big(
\hat{\mu}_{p,*}(t)-\mu_*
\le
-C_{p,*}(t)
\Big)
&\le
\frac1{t^4}. \label{eqn:Event 2 Probability}
\end{align}
\end{subequations}
Furthermore, by the choice of $l$ in
\eqref{eq:l_group}, event \eqref{eqn:Event3} cannot occur whenever
$N_{p,k}^{+}(t)>l$.
Hence at least one of the events
\eqref{eqn:Event 1}
or
\eqref{eqn:Event 2}
must occur.

Consequently,
\begin{align}
&
\mathbf{1}
\left\{
A_p(t)=k,\,
N_{p,k}^{+}(t)>l
\right\}
\nonumber\\
&\le
\mathbf{1}
\left\{
\hat{\mu}_{p,*}(t)
\le
\mu_*
-
C_{p,*}(t)
\right\}
\nonumber\\
&\quad+
\mathbf{1}
\left\{
\hat{\mu}_{p,k}(t)
\ge
\mu_k
+
C_{p,k}(t)
\right\}.
\label{eq:indicator_decomposition}
\end{align}

Taking expectations and using \eqref{eqn:Event 1 Probability} and \eqref{eqn:Event 2 Probability},
\begin{align}
&
\mathbb{E}
\!\left[
\mathbf{1}
\left\{
A_p(t)=k,\,
N_{p,k}^{+}(t)>l
\right\}
\right]
\nonumber\\
&\le
\frac1{t^4}
+
\frac1{t^4}
=
\frac{2}{t^4}.
\label{eq:pull_probability_bound}
\end{align}

Taking expectation in \eqref{eq:nonbuy_split} and using \eqref{eq:type1_bound}, \eqref{eq:type2_bound} , \eqref{eq:pull_probability_bound}

\begin{align}
&
\mathbb{E}
\!\left[
\sum_{t=1}^{T}
\sum_{p=1}^{M}
\mathbf{1}\{A_p(t)=k\}
\right]
\nonumber\\
&\stackrel{(a)}{\le}
\min\{M,B+1\}\,(l+1)
+\frac{4MN\alpha^{2}}
{2\alpha^{2}-1}+ \nonumber \\
& \quad \quad \sum_{t=\min\{M,B+1\}l+1}^{T}
\sum_{p=1}^{M}
\frac{2}{t^4}
\nonumber\\
&\le
\min\{M,B+1\}\,(l+1)
+\frac{4MN\alpha^{2}}
{2\alpha^{2}-1}+ \nonumber \\
& \quad \quad 2M
\int_{\min\{M,B+1\}l}^{\infty}
\frac1{t^4}\,dt
\nonumber\\
&=
\min\{M,B+1\}\,(l+1)
+\frac{4MN\alpha^{2}}
{2\alpha^{2}-1}+ \nonumber \\
& \quad \quad \frac{2M}{3}
\left(
\frac1{\min\{M,B+1\}l}
\right)^3
\nonumber\\
&\le
\min\{M,B+1\}\,(l+1)
+\frac{4MN\alpha^{2}}
{2\alpha^{2}-1}+
\frac{2M}{3}.
\label{eq:nonbuy_final_group}
\end{align}
Therefore from \eqref{eq:nonbuy_final_group}
\begin{align*}
\mathbb{E}
\!\left[
n_k^{\mathrm{group}}(T+1)
\right]
&\le
\min\{M,B+1\}\,(l+1)
+\frac{4MN\alpha^{2}}
{2\alpha^{2}-1}+
\frac{2M}{3}
\nonumber\\.
\end{align*}
Substituting
\eqref{eq:l_group},
\begin{align}
\mathbb{E}
\!\left[
n_k^{\mathrm{group}}(T+1)
\right]
&\le
\min\{M,B+1\}
\left(
\frac{8\log T}{\Delta_k^2}
+2
\right) \nonumber \\
&+ \frac{4MN\alpha^{2}}
{2\alpha^{2}-1} +
\frac{2M}{3}.
\label{eq:group_count_final}
\end{align}
Finally, from Lemma 1
\begin{align}
\mathbb{E}
\!\left[
R_{\mathrm{group}}^{\mathrm{pull}}(T)
\right]
=
\sum_{k:\Delta_k>0}
\Delta_k
\,
\mathbb{E}
\!\left[
n_k^{\mathrm{group}}(T+1)
\right],
\end{align}
and substituting \eqref{eq:group_count_final}, we obtain
\begin{align}
\mathbb{E}
\!\left[
R_{\mathrm{group}}^{\mathrm{pull}}(T)
\right]
\le
\sum_{k:\Delta_k>0}
\Delta_k
\Bigg[
&
\min\{M,B+1\} \times \nonumber\\
&\left(
\frac{8\log T}{\Delta_k^2}
+2
\right)+ \frac{4MN\alpha^{2}}
{2\alpha^{2}-1}+
\frac{2M}{3}
\Bigg].
\end{align}
This completes the proof.

\end{proof}

\section{Proof of Lemma 2 (Main Paper)}
\label{sec:proof_lemma2}

\noindent\textbf{Lemma 2.}
Let
$B_{p,k}(T)$ denote the number of times player $p$
purchases information regarding arm $k$ up to time $T$.
For every sub-optimal arm $k$ with $\Delta_k>0$ and for $\alpha > 1/\sqrt{2}$, define
\begin{align*}
U_k(T)
=
\frac{
\min\{M,B+1\}
\left(
\dfrac{8\log T}{\Delta_k^2}+1
\right)
+\dfrac{4MN\alpha^2}{2\alpha^2-1}
+\dfrac{2M}{3}
}{M}.
\end{align*}

Then
\begin{align*}
\mathbb{E}[B_{p,k}(T)]
\le
1+
\frac{\log U_k(T)}
{\log(1+B)}.
\end{align*}

\begin{proof}
Fix a player $p$ and a sub-optimal arm $k$.
Let
\[
B_{p,k}(T)
=
\sum_{t=1}^{T}
\gamma_{p,k}(t)
\]
denote the total number of purchases of information
regarding arm $k$ made by player $p$ up to time $T$. Suppose that player $p$ purchases information for arm $k$
for the $j$-th time at some time instant $t$. Immediately before the purchase, the buying condition is
satisfied, i.e.,

\begin{align}
\sum_{q\neq p}
n_{q,k}(t)
-
\sum_{q\neq p}
n_{q,k}(\tau_{p,k}^{t-1})
\ge
B\,N_{p,k}(t).
\label{eq:buy_cond_proof}
\end{align}

After the purchase, all newly accumulated observations are
incorporated into the effective sample count.
Consequently,

\begin{align}
N_{p,k}^{+}(t)
&=
N_{p,k}(t)
+
\Bigg(
\sum_{q\neq p}
n_{q,k}(t)
-
\sum_{q\neq p}
n_{q,k}(\tau_{p,k}^{t-1})
\Bigg)
\nonumber\\
&\ge
N_{p,k}(t)
+
B\,N_{p,k}(t)
\nonumber\\
&=
(1+B)\,N_{p,k}(t).
\label{eq:multiplicative_growth}
\end{align}

Therefore, every purchase increases the effective sample
count by at least a multiplicative factor $(1+B)$.

Since each player initially pulls every arm once,

\[
N_{p,k}^{+}(t)\ge M.
\]

Applying \eqref{eq:multiplicative_growth}
recursively, after the $j$-th purchase we obtain

\begin{align}
N_{p,k}^{+}(t)
\ge
M(1+B)^{j-1}.
\label{eq:recursive_growth}
\end{align}

Since the effective sample count is non-decreasing in time,

\[
N_{p,k}^{+}(T)
\ge
M(1+B)^{j-1}.
\]

Rearranging,

\begin{align}
j
&\le
1+
\frac{
\log
\!\left(
\dfrac{N_{p,k}^{+}(T)}{M}
\right)
}{
\log(1+B)
}.
\label{eq:j_bound}
\end{align}

Since $B_{p,k}(T)$ equals the total number of purchases,

\begin{align}
B_{p,k}(T)
\le
1+
\frac{
\log
\!\left(
\dfrac{
N_{p,k}^{+}(T)}{M}
\right)
}{
\log(1+B)
}.
\label{eq:buy_count_random}
\end{align}

Next, observe that the effective sample count cannot exceed
the total number of observations collected by all players.
Hence

\begin{align}
N_{p,k}^{+}(t)
\le
n_k^{\mathrm{group}}(T+1).
\label{eq:effective_le_group}
\end{align}

Substituting \eqref{eq:effective_le_group}
into \eqref{eq:buy_count_random} gives

\begin{align}
B_{p,k}(T)
\le
1+
\frac{
\log
\!\left(
\dfrac{
n_k^{\mathrm{group}}(T+1)
}{M}
\right)
}{
\log(1+B)
}.
\label{eq:buy_count_group}
\end{align}

Taking expectation on both sides,

\begin{align}
\mathbb{E}[B_{p,k}(T)]
&\le
1+
\frac{
\mathbb{E}
\!\left[
\log
\!\left(
\dfrac{
n_k^{\mathrm{group}}(T+1)
}{M}
\right)
\right]
}{
\log(1+B)
}.
\label{eq:expectation_step}
\end{align}

Since $\log(\cdot)$ is concave,
Jensen's inequality implies

\begin{align}
\mathbb{E}
\!\left[
\log
\!\left(
\dfrac{
n_k^{\mathrm{group}}(T+1)
}{M}
\right)
\right]
\le
\log
\!\left(
\dfrac{
\mathbb{E}
\!\left[
n_k^{\mathrm{group}}(T+1)
\right]
}{M}
\right).
\label{eq:jensen_step}
\end{align}

Substituting \eqref{eq:jensen_step}
into \eqref{eq:expectation_step},

\begin{align}
\mathbb{E}[B_{p,k}(T)]
\le
1+
\frac{
\log
\!\left(
\dfrac{
\mathbb{E}
\!\left[
n_k^{\mathrm{group}}(T+1)
\right]
}{M}
\right)
}{
\log(1+B)
}.
\label{eq:before_substitution}
\end{align}

From \ref{eq:group_count_final},

\begin{align}
\mathbb{E}
\!\left[
n_k^{\mathrm{group}}(T+1)
\right]
\le
M U_k(T).
\label{eq:group_count_bound}
\end{align}

Substituting \eqref{eq:group_count_bound}
into \eqref{eq:before_substitution} yields

\begin{align}
\mathbb{E}[B_{p,k}(T)]
\le
1+
\frac{
\log
\!\left(
U_k(T)
\right)
}{
\log(1+B)
}.
\end{align}

This completes the proof.

\end{proof}

\section{Proof of Lemma 3 (Main Paper)}
\label{sec:proof_lemma3}

\noindent\textbf{Lemma 3.}
Let $k^\star$ denote the optimal arm.
Then, for $\alpha > 1/\sqrt{2}$,
\begin{align*}
\mathbb{E}[B_{p,k^\star}(T)]
\le
1+
\frac{\log T}{\log(1+B)}.
\end{align*}

\begin{proof}

Using the same geometric-growth argument as in the proof of Lemma 2, \eqref{eq:buy_count_group} applies for optimal arm also. In particular,

\begin{align}
B_{p,k^\star}(T)
\le
1+
\frac{
\log
\!\left(
\dfrac{
n_{k^\star}^{\mathrm{group}}(T+1)
}{M}
\right)
}{
\log(1+B)
}.
\label{eq:optimal_buy_count_group}
\end{align}

Taking expectation on both sides and applying Jensen's inequality gives

\begin{align}
\mathbb{E}[B_{p,k^\star}(T)]
\le
1+
\frac{
\log
\!\left(
\dfrac{
\mathbb{E}
\!\left[
n_{k^\star}^{\mathrm{group}}(T+1)
\right]
}{M}
\right)
}{
\log(1+B)
}.
\end{align}

Since every player pulls exactly one arm in each round,

\[
n_{k^\star}^{\mathrm{group}}(T+1)\le MT.
\]

Hence,

\[
\mathbb{E}
\!\left[
n_{k^\star}^{\mathrm{group}}(T+1)
\right]
\le MT.
\]

Substituting the above bound,

\begin{align}
\mathbb{E}[B_{p,k^\star}(T)]
\le
1+
\frac{\log T}{\log(1+B)}.
\label{eq:optimal_buy_log}
\end{align}

\end{proof}

\section{Proof of Theorem 2 (Main Paper)}
\label{sec:proof_theorem2}

\noindent\textbf{Theorem 2.}
Under Algorithm~1, for $\alpha > 1/\sqrt{2}$, the expected buying
cost of player $p$ satisfies
\begin{align}
\mathbb{E}
\!\left[
R_p^{\mathrm{buy}}(T)
\right]
\le
&
B
\left(
1+
\frac{\log T}{\log(1+B)}
\right)
\nonumber\\
&+
(N-1)B
\left(
1+
\frac{\log U_{\max}(T)}
{\log(1+B)}
\right),
\end{align}
where
\[
U_{\max}(T)=\max_{k:\Delta_k>0}U_k(T).
\]
Furthermore,
\[
\mathbb{E}
\!\left[
R_p^{\mathrm{buy}}(T)
\right]
=
O
\!\left(
B\log T
\right),
\]
and
\[
\mathbb{E}
\!\left[
R_{\mathrm{group}}^{\mathrm{buy}}(T)
\right]
=
O
\!\left(
MB\log T
\right).
\]

\begin{proof}

Since each purchase incurs cost $B$, the buying cost of
player $p$ can be decomposed as

\begin{align}
R_p^{\mathrm{buy}}(T)
=
B\,B_{p,k^\star}(T)
+
B
\sum_{k:\Delta_k>0}
B_{p,k}(T).
\end{align}

Taking expectation on both sides gives

\begin{align}
\mathbb{E}
\!\left[
R_p^{\mathrm{buy}}(T)
\right]
&=
B\,\mathbb{E}[B_{p,k^\star}(T)]
+
B
\sum_{k:\Delta_k>0}
\mathbb{E}[B_{p,k}(T)].
\label{eq:buy_regret_decomp}
\end{align}

Applying Lemma 2 and
3, we obtain

\begin{align}
\mathbb{E}
\!\left[
R_p^{\mathrm{buy}}(T)
\right]
\le
&
B
\left(
1+
\frac{\log T}{\log(1+B)}
\right)
\nonumber\\
&+
B
\sum_{k:\Delta_k>0}
\left(
1+
\frac{\log U_k(T)}
{\log(1+B)}
\right).
\end{align}

Using the definition of $\Delta_{\min}$,

\begin{align}
U_k(T)
\le
U_{\max}(T),
\end{align}

for every sub-optimal arm $k$.
Therefore,

\begin{align}
\mathbb{E}
\!\left[
R_p^{\mathrm{buy}}(T)
\right]
\le
&
B
\left(
1+
\frac{\log T}{\log(1+B)}
\right)
\nonumber\\
&+
(N-1)B
\left(
1+
\frac{\log U_{\max}(T)}
{\log(1+B)}
\right).
\end{align}
Hence,

\begin{align}
\mathbb{E}
\!\left[
R_p^{\mathrm{buy}}(T)
\right]
=
O
\!\left(
B\log T
\right).
\end{align}

Summing over all $M$ players yields

\begin{align}
\mathbb{E}
\!\left[
R_{\mathrm{group}}^{\mathrm{buy}}(T)
\right]
=
O
\!\left(
MB\log T
\right).
\end{align}

This completes the proof.

\end{proof}

\section{Proof of Theorem 3 (Main Paper)}
\label{sec:proof_theorem3}

\noindent\textbf{Theorem 3.}
The total expected group regret of Algorithm~1 for
$\alpha > 1/\sqrt{2}$ satisfies
\[
\mathbb{E}
\!\left[
R_{\mathrm{group}}(T)
\right]
=
O\!\left(
N\min\{M,B+1\}\log T
+
MNB\log\log T
\right).
\]

\begin{proof}
Using the decomposition

\begin{align}
R_{\mathrm{group}}(T)
=
R_{\mathrm{group}}^{\mathrm{pull}}(T)
+
R_{\mathrm{group}}^{\mathrm{buy}}(T),
\end{align}

and applying Theorems 1
and 2,
the stated finite-time bound follows immediately.
The asymptotic order is obtained by combining the
corresponding orders from these two theorems.
\end{proof}

\end{document}